\documentclass{article}

\usepackage{microtype}
\usepackage{graphicx}
\usepackage{booktabs} 

\usepackage{hyperref}

\usepackage[accepted]{arxiv_style/arxiv}

\usepackage{amsmath}
\usepackage{amssymb}
\usepackage{mathtools}
\usepackage{amsthm}

\usepackage[capitalize,noabbrev]{cleveref}

\usepackage{amsmath,amsfonts,amssymb,amsthm}
\usepackage{bbm}
\usepackage{bm}
\usepackage{xcolor}
\usepackage{hyperref}
\hypersetup{
    colorlinks = true,
    citecolor = blue,
}
\usepackage{algorithm}
\usepackage{algorithmic}
\renewcommand{\algorithmiccomment}[1]{\bgroup\hfill$\triangleright$~#1\egroup}
\usepackage{thmtools,thm-restate}
\usepackage{natbib}
\usepackage{booktabs}
\usepackage{subcaption}
\usepackage{graphicx}
\usepackage{cancel}

\newtheorem{fact}{Fact}
\newtheorem{defn}{Definition}
\newtheorem{example}{Example}

\newtheorem{thm}{Theorem}
\newtheorem*{thm*}{Theorem}
\newtheorem{lemma}{Lemma}
\newtheorem{prop}{Proposition}

\newtheorem{assumption}{Assumption}
\newtheorem{remark}{Remark}

\newtheorem*{question*}{Question}
\newtheorem*{answer*}{Answer}
\newtheorem*{solution*}{Solution}
\newtheorem*{nextstep*}{Next Step}
\newtheorem*{issue*}{Issue}

\usepackage{prettyref}
\newcommand{\rref}[2][]{\prettyref{#2}}
\newrefformat{model}{Model\,\ref{#1}}
\newrefformat{fact}{Fact\,\ref{#1}}
\newrefformat{listing}{Listing\,\ref{#1}}
\newrefformat{alg}{Algorithm\,\ref{#1}}
\newrefformat{algm}{Algorithm\,\ref{#1}}
\newrefformat{line}{Line\,\ref{#1}}
\newrefformat{sec}{Section\,\ref{#1}}
\newrefformat{subsec}{Subsection\,\ref{#1}}
\newrefformat{section}{Section\,\ref{#1}}
\newrefformat{appendix}{Appendix\,\ref{#1}}
\newrefformat{app}{Appendix\,\ref{#1}}
\newrefformat{def}{Definition\,\ref{#1}}
\newrefformat{defn}{Definition\,\ref{#1}}
\newrefformat{thm}{Theorem\,\ref{#1}}
\newrefformat{ax}{\ref{#1}}
\newrefformat{prop}{Prop.\,\ref{#1}}
\newrefformat{lemma}{Lemma\,\ref{#1}}
\newrefformat{cor}{Corollary\,\ref{#1}}
\newrefformat{corollary}{Corollary\,\ref{#1}}
\newrefformat{conj}{Conjecture\,\ref{#1}}
\newrefformat{ex}{Example\,\ref{#1}}
\newrefformat{tab}{Table\,\ref{#1}}
\newrefformat{fig}{Fig.\,\ref{#1}}
\newrefformat{eqn}{Equation~(\ref{#1})}
\newrefformat{problem}{Problem\,\ref{#1}}
\newrefformat{assumption}{Assumption\,\ref{#1}}
\newrefformat{claim}{Claim\,\ref{#1}}
\newrefformat{remark}{Remark\,\ref{#1}}

\newcommand{\Var}{\textrm{Var}}
\newcommand{\Cov}{\textrm{Cov}}

\newcommand{\Unif}{\mathsf{Unif}}

\newcommand\independent{\protect\mathpalette{\protect\independenT}{\perp}}
\def\independenT#1#2{\mathrel{\rlap{$#1#2$}\mkern2mu{#1#2}}}

\newcommand{\kron}{\mathbbm{1}}

\newcommand{\inv}{{-1}}

\DeclareMathOperator*{\argmax}{arg\,max}

\DeclareMathOperator{\pa}{pa}

\newcommand{\Pa}{\overline{\textrm{pa}}}
\newcommand{\an}{{\textrm{an}}}

\newcommand{\T}{^\top}
\newcommand{\minusT}{{^{-\top}}}

\newcommand{\bbE}{\mathbb{E}}

\newcommand{\bbR}{\mathbb{R}}

\newcommand{\bbP}{\mathbb{P}}

\newcommand{\hatA}{\widehat{A}}
\newcommand{\hatB}{\widehat{B}}
\newcommand{\hatC}{\widehat{C}}
\newcommand{\hatD}{\widehat{D}}
\newcommand{\hatH}{\widehat{H}}
\newcommand{\hatM}{\widehat{M}}
\newcommand{\hatQ}{\widehat{Q}}
\newcommand{\hatR}{\widehat{R}}

\newcommand{\hati}{\widehat{i}}

\newcommand{\hatt}{\widehat{t}}

\newcommand{\hLambda}{\widehat{\Lambda}}

\newcommand{\bc}{{\mathbf{c}}}

\newcommand{\be}{{\mathbf{e}}}

\newcommand{\bh}{{\mathbf{h}}}

\newcommand{\bq}{{\mathbf{q}}}
\newcommand{\br}{{\mathbf{r}}}

\newcommand{\bv}{{\mathbf{v}}}

\newcommand{\boldalpha}{{\bm{\alpha}}}

\newcommand{\boldeta}{{\bm{\eta}}}

\newcommand{\cA}{{\mathcal{A}}}

\newcommand{\cG}{{\mathcal{G}}}

\newcommand{\cI}{{\mathcal{I}}}

\newcommand{\cN}{{\mathcal{N}}}
\newcommand{\cO}{{\mathcal{O}}}

\newcommand{\cX}{{\mathcal{X}}}

\newcommand{\hatG}{\widehat{G}}
\newcommand{\hatOmega}{\widehat{\Omega}}

\newcommand{\rank}{{\textnormal{rank}}}
\newcommand{\rowspan}{{\textnormal{rowspan}}}
\newcommand{\proj}{{\textnormal{proj}}}

\newcommand{\bzero}{\mathbf{0}}

\newcommand{\IterativeDifferenceProjection}{\textsc{IterativeDifferenceProjection}}
\newcommand{\IdentifyPartialOrder}{\textsc{ID-PartialOrder}}
\newcommand{\IdentifyAncestors}{\textsc{ID-Ancestors}}

\newcommand{\IterativeDifferenceProjectionNoisy}{\textsc{IterativeDifferenceProjectionFiniteSample}}
\newcommand{\IdentifyPartialOrderNoisy}{\textsc{IdentifyPartialOrderFiniteSample}}
\newcommand{\IdentifyAncestorsNoisy}{\textsc{IdentifyAncestorsFiniteSample}}

\newcommand{\cholesky}{\textsc{cholesky}}
\newcommand{\VEC}{\textnormal{vec}}
\newcommand{\normalize}{\textnormal{normalize}}

\newcommand{\tTheta}{\tilde{\Theta}}
\newcommand{\hatbq}{\widehat{\bq}}
\newcommand{\An}{\overline{\an}}
\newcommand{\Sig}{\textnormal{Sig}}
\newcommand{\source}{\textnormal{source}}
\newcommand{\Bsigma}{B^{(\sigma)}}
\newcommand{\Hsigma}{H^{(\sigma)}}
\newcommand{\Csigma}{C^{(\sigma)}}
\newcommand{\barcG}{\overline{\cG}}
\newcommand{\hatcG}{\widehat{\cG}}
\newcommand{\hTheta}{\widehat{\Theta}}
\newcommand{\tildeB}{\tilde{B}}
\newcommand{\st}{\textnormal{s.t.}}

\icmltitlerunning{Linear Causal Disentanglement via Interventions}

\begin{document}

\twocolumn[
\icmltitle{Linear Causal Disentanglement via Interventions}


\icmlsetsymbol{equal}{*}

\begin{icmlauthorlist}
\icmlauthor{Chandler Squires}{equal,broad,lids}
\icmlauthor{Anna Seigal}{equal,broad,seas}
\icmlauthor{Salil Bhate}{broad}
\icmlauthor{Caroline Uhler}{broad,lids}
\end{icmlauthorlist}

\icmlaffiliation{lids}{Laboratory for Information and Decision Systems, MIT}
\icmlaffiliation{seas}{School of Engineering and Applied Sciences, Harvard University}
\icmlaffiliation{broad}{Broad Institute of MIT and Harvard}

\icmlcorrespondingauthor{Chandler Squires}{csquires@mit.edu}

\icmlkeywords{Causal Structure Learning, Causal Representation Learning}

\vskip 0.3in
]


\printAffiliationsAndNotice{\icmlEqualContribution} 

\begin{abstract}
Causal disentanglement seeks a representation of data involving latent variables that are related via a causal model. 
A representation is identifiable if both the latent model and the transformation from latent to observed variables are unique.
In this paper, we study observed variables that are a linear transformation of a linear latent causal model.
Data from interventions are necessary for identifiability: if one latent variable is missing an intervention, we show that there exist distinct models that cannot be distinguished.
Conversely, we show that a single intervention on each latent variable is sufficient for identifiability.
Our proof uses a generalization of the RQ decomposition of a matrix that replaces the usual orthogonal and upper triangular conditions with analogues depending on a partial order on the rows of the matrix, with partial order determined by a latent causal model.
We corroborate our theoretical results with a method for causal disentanglement. 
We show that the method accurately recovers a latent causal model on synthetic and semi-synthetic data and we illustrate a use case on a dataset of single-cell RNA sequencing measurements.


\end{abstract}

\addtocontents{toc}{\protect\setcounter{tocdepth}{0}}
\section{Introduction}\label{sec:intro}

The goal of representation learning is to find a description of data that is interpretable, useful for reasoning, and generalizable.
Such a representation \textit{disentangles} the data into conceptually distinct variables.
Traditionally, conceptual distinctness of variables has meant statistical independence.
This is the setting of independent component analysis \citep{comon1994independent}.
However, human reasoning often involves variables that are not statistically independent. 
For example, the presence of a sink and the presence of a mirror in an image.
It is therefore natural to generalize conceptual distinctness to variables that are \textit{causally autonomous}; i.e., interventions can be performed on each variable separately.
This motivates \textit{causal disentanglement} \citep{yang2021causalvae}, the recovery of causally autonomous variables from data.

In this paper, we study the identifiability of causal disentanglement; i.e., 
its uniqueness. We adopt a generative perspective, as in \citep{bengio2013representation,moran2021identifiable}.
We assume that the observed variables are generated in two steps.
First, latent variables $Z$ are sampled from a distribution $\bbP(Z)$.
Then, the observed variables $X$ are the image of the latent variables under a deterministic mixing function.
We assume that the latent variables are generated according to a linear structural equation model and that the mixing function is an injective linear map.
Recent work has studied identifiability of various settings in representation learning \citep{khemakhem2020variational,ahuja2021properties}.
A common assumption for identifiability is that variables are observed across multiple contexts, each affecting the latent distribution $\bbP(Z)$ but not the mixing function.
In our setup, each context is either an \textit{intervention} on a latent variable, or is \textit{observational}, i.e., has no interventions.
We use the same terminology for interventions as \citet{squires22causal}.
From most to least general, a \textit{soft} intervention on $Z_i$ changes the dependency of $Z_i$ on its direct causes, a \textit{perfect} intervention removes this dependency but allows for stochasticity of $Z_i$, and a \textit{do}-intervention sets $Z_i$ to a deterministic value.

Our main result is that our linear causal disentanglement setup is identifiable if, in addition to an observational context, for each latent variable $Z_i$, there is a context where $Z_i$ is the 
intervened variable under a perfect intervention; see \rref{sec:theoretical-sufficiency}.
This is a \textit{sufficient} condition for identifiability.
Furthermore, we show that the condition of at least one intervention per latent node is \textit{necessary} in the worst case: if some latent node is not intervened in any context, then there exist latent causal representations that are not identifiable; see \rref{sec:theoretical-necessity}.
Our focus in this paper is on identifiability guarantees.
Nonetheless, we convert our proofs into a method for causal disentanglement in the finite-sample setting.
In \rref{sec:experimental}, we apply the method to synthetic and semi-synthetic data and show that it recovers the generative model, and we compute a linear causal disentanglement on a single-cell RNA sequencing dataset.

\subsection{Motivating Example}

Consider two latent variables $Z = (Z_1, Z_2)$.
Assume that $X=(X_1, X_2)$ is observed in two contexts $k \in \{ 0, 1\}$, that $X = G Z$ in both contexts, and that in context $k$,
{
\setlength{\abovedisplayskip}{5pt}
\setlength{\belowdisplayskip}{5pt}
\begin{equation*}
    Z = A_k Z + \Omega_k^{1/2} \varepsilon \quad \text{for} \quad 
\varepsilon \sim \cN(0, I).
\end{equation*}
Let
\begin{equation*}
\begin{aligned}
    A_0 
    = A_1 =
    &\begin{bmatrix}
    0 & -1
    \\
    0 & \phantom{-}0
    \end{bmatrix},
    \quad
    \Omega_0 
    = 
    \begin{bmatrix}
    1 & 0
    \\
    0 & 1
    \end{bmatrix},
    \\
    \Omega_1
    = 
    &\begin{bmatrix}
    1 & 0
    \\
    0 & 1/4
    \end{bmatrix},
    \quad
    G
    = 
    \begin{bmatrix}
    -2 & \phantom{-}2
    \\
    \phantom{-}2 & -1
    \end{bmatrix}.
\end{aligned}
\end{equation*}

Context $k=1$ is an intervention on $Z_2$ that changes its variance.
}
The covariance of $X$ in contexts $k = 0$ and $k=1$ are, respectively,
\[
\Sigma_0 
= 
\begin{bmatrix}
\phantom{-}20 & -16
\\
-16 & \phantom{-}13
\end{bmatrix}
\qquad \text{and} \qquad 
\Sigma_1
= 
\begin{bmatrix}
\phantom{-}8 & -7
\\
-7 & 25/4
\end{bmatrix},
\]
since $\Sigma_k = G(I - A_k)^\inv \Omega_k (I - A_k)\minusT G\T$.
However, the following parameters give the same covariance matrices:
\begin{align*}
\hatA_0 
= 
\hatA_1
=
\begin{bmatrix}
0 & 0
\\
0 & 0
\end{bmatrix},
\quad
\hatOmega_0 
&= 
\begin{bmatrix}
1 & 0
\\
0 & 1
\end{bmatrix},
\\
\hatOmega_1
= 
\begin{bmatrix}
1 & 0
\\
0 & 1/4
\end{bmatrix},
\quad
\hatG
&= 
\begin{bmatrix}
-2 & \phantom{-}4
\\
\phantom{-}2 & -3
\end{bmatrix}.
\end{align*}
The second set of parameters imply independence of $Z_1$ and $Z_2$, since $(\hatA_0)_{1,2} = 0$, whereas the original parameters imply non-independence since $(A_0)_{1,2} \neq 0$.
This non-identifiability holds for generic $A_0$, $A_1$, $\Omega_0$, $\Omega_1$, and $G$, where $A_1$ comes from an intervention on $Z_2$, see \rref{sec:theoretical-necessity}.
This non-identifiability extends to any number of latent variables $d$: we show that, in the worst case, non-identifiability holds when fewer than $d+1$ contexts are observed.

\subsection{Related Work}\label{sec:related}

The growing field of causal representation learning blends techniques from multiple lines of work.
Chief among these are identifiable representation learning, causal structure learning, and latent DAG (directed acyclic graph) learning.

\textbf{Identifiable representation learning.}
The identifiability of the linear independent component analysis (ICA) model was given in \citet{comon1994independent}. 
Identifiability of a nonlinear ICA model is studied in \citep{hyvarinen2019nonlinear,khemakhem2020variational}, 
in the presence of auxiliary variables.
The ICA model imposes the stringent condition that the latent variables are independent (in the linear setting) or conditionally independent given the auxiliary variables (in the nonlinear setting).
Recent works on identifiable representation learning \citep{ahuja2021properties,zimmermann2021contrastive} introduce structure on the data generating process to circumvent the independence condition.
However, they do not consider latent variables that are causally related.

\begin{table*}[t]
    \centering
    \begin{tabular}{@{}r|lll@{}}
    \toprule
        & Setting 
        & 
        Graphical Conditions 
        &
        Identification Result
        \\
        \hline
        \hline
        \citet{silva2006learning} 
        & 
        \textbf{LNG} 
        & 
        All children pure
        & 
        Identified up to Markov equivalence.
        \\
        \hline
        \citet{halpern2015anchored} 
        & 
        \textbf{Dis} 
        & 
        1 pure child per latent
        & 
        Identified up to Markov equivalence.
        \\
        \hline
        \citet{cai2019triad} 
        & \textbf{LNG} 
        & 
        2 pure children per latent 
        & 
        Fully identified.
        \\
        \hline
        \citet{xie2020generalized} 
        & 
        \textbf{LNG} 
        & 
        2 pure children per latent 
        & 
        Fully identified.
        \\
        \hline
        \citet{xie2022identification} 
        & 
        \textbf{LNG} 
        & 
        2 pure children per latent*
        & 
        Fully identified.
        \\
        \hline
        \citet{kivva2021learning} 
        & 
        \textbf{Dis} 
        & 
        No twins
        & 
        Identified up to Markov equivalence.
        \\
        \hline
        \citet{liu2022weight}
        & 
        \textbf{LG} 
        & 
        None
        & 
        Fully identified from 2$|V|$
        \\
        &&&
        perfect interventions if $|E| \leq |V|$.
        \\
        \hline
        \citet{ahuja2022interventional}
        & 
        \textbf{Poly} 
        & 
        None
        & 
        Fully identified from $|V|$ do-interventions.
        \\
        \hline
        \textcolor{violet}{This paper} 
        & 
        \textbf{LNG} or \textbf{LG} 
        & 
        None
        &
        Fully identified from $|V|$ perfect interventions.
        \\
    \bottomrule
    \end{tabular}
    \caption{
    \textbf{Settings from prior works on learning latent DAG models.}
    \textbf{LNG} is short for \textit{linear non-Gaussian}, \textbf{LG} for \textit{linear Gaussian}, \textbf{Dis} for \textbf{discrete}, and \textbf{Poly} for \textit{polynomial mixing}.
    In *, pure children are allowed to be latent.
    The number of nodes and the number of edges in the latent graph are denoted $|V|$ and $|E|$, respectively.
    }
    \label{tab:prior-work}
    \vspace{-0.2in}
\end{table*}

\textbf{Causal Structure Learning.}
Causal structure is identifiable up to an \textit{equivalence class} that depends on the available interventional data \citep{verma1990equivalence,hauser2012characterization,yang2018characterizing,squires2020permutation,jaber2020causal}.
See \citet{squires22causal} for a recent review.
A key line of work \citep{eberhardt2005number,hyttinen2013experiment} characterizes the interventions necessary and sufficient to ensure that the causal structure is fully identifiable; i.e., that the equivalence class is of size one.
In particular, \citet{eberhardt2005number} showed that $d - 1$ interventions are in the worst case necessary to fully identify a causal DAG model on $d$ nodes.
The current paper extends this line of work to DAG models over \textit{latent} variables.

\textbf{Learning latent DAG models.}
The task of learning a DAG over latent variables dates back at least to \citet{silva2006learning}.
They introduced the notion of a \textit{pure child}: an observed variable $X_i$ with only one latent parent, such $X_i$ is also called an \textit{anchor} \citep{halpern2015anchored,saeed2020anchored}.
The method of \citet{silva2006learning} requires that all observed variables are pure children.
Recent works relax this assumption by studying the linear non-Gaussian setting, where all latent and observed variables are linear functions of their parents plus independent non-Gaussian noise.
For example, \citet{cai2019triad} propose a method which learns a latent DAG under the assumption that each latent variable has at least two pure children.
The pure child assumption can be extended to allow subsets of latent variables with the same observed children, as in \citet{xie2020generalized}, which introduces the Generalized Independent Noise condition.
This condition was used by \citet{xie2022identification} to permit latent variables with no observed children; i.e., a hierarchical latent model.

Other works consider the discrete setting, requiring that the latent and observed variables are discrete \citep{halpern2015anchored} or that the latent variables are discrete \citep{kivva2021learning}.
The paper \citet{kivva2021learning} relaxes the pure child assumption, as follows. The children of node $Z_i$ are the variables with a directed edge from $Z_i$. 
The \textit{no twins} assumption says that the observed children of any two latent nodes are distinct sets.
A similar assumption called \textit{strong non-redundancy} appears in \citet{adams2021identification}, which considers models whose latent variables can be downstream of observed variables.
See \rref{appendix:additional-related}.
These works require sparsity in the map between latent and observed variables: they do not allow all observed variables to be children of all latent variables, which is the setting of the present paper.

A number of recent works discard the sparsity requirement.
\citet{ahuja2022weakly} and \citet{brehmer2022weakly} learn a latent DAG from paired counterfactual data.
%
In contrast, we study \textit{unpaired} data, which is more realistic in applications such as biology \citep{stark2020scim}.
To the best of our knowledge, only two works consider unpaired data without sparsity assumptions.
\citet{liu2022weight} study a linear Gaussian model over $d$ latent variables, a nonlinear mixing function, and vector-valued contexts.
Their identifiability result only applies to our setting if the latent graph has at most as many edges as nodes, see \rref{appendix:weight-variant-comparison}.
In that case, their result implies that $2 d$ interventions suffice for identifiability.
We strengthen this result, showing that (i) $d$ interventions are sufficient and (ii) no restrictions on the latent graph are required.
Moreover, we show that $d$ interventions are, in the worst case, necessary. 
Such necessary conditions do not appear in prior work on identifying latent DAGs.
Finally, contemporaneous work \citep{ahuja2022interventional} shows that a latent DAG is identifiable from the more restricted class of do-interventions, but allow non-linear relationships.
See \rref{tab:prior-work} for a summary of prior work.

\section{Setup}\label{sec:setup}

We consider $d$ latent variables $Z = (Z_1, \ldots, Z_d)$, generated according to a linear structural equation model.
We index contexts by $k \in \{ 0 \} \cup [K]$, where $[K] := \{1, \ldots, K \}$.
The linear structural equation models in each context are related: context $k = 0$ is \textit{observational} data, while contexts $k \in [K]$ are \textit{interventional} data.
We now state the assumptions for our model; see also \rref{fig:model}.

\begin{assumption}\label{assumption:data-generating}
\hfill
{
\setlength{\abovedisplayskip}{7pt}
\setlength{\belowdisplayskip}{5pt}
\begin{itemize}
    \item[(a)] \textbf{Linear latent model}: Let $\cG$ be a DAG with nodes ordered so that an edge $j \to i$ implies $j > i$.
    The latent variables $Z$ follow a linear structural equation model: in context $k$, the latent variables $Z$ satisfy
    \begin{equation*}\label{eqn:z-structural-equation}
        Z = A_k Z + \Omega_k^{1/2} \varepsilon,
        \quad\quad
        \Cov(\varepsilon) = I_d,
    \end{equation*}
    where $I_d \in \bbR^{d \times d}$ is the identity matrix, $\Omega_k \in \bbR^{d \times d}$ is diagonal with positive entries, and $A_k \in \bbR^{d \times d}$ has $(A_k)_{ij} \neq 0$ if and only if there is an edge $j \to i$ in $\cG$.
    That is, in context $k$,
    \begin{equation}\label{eqn:z-equation-b}
        Z = B_k^\inv \varepsilon,
        \quad
        \textnormal{where} \quad B_k = \Omega_k^{-1/2} (I_d - A_k).
    \end{equation}
    
    \item[(b)] \textbf{Generic single-node interventions:} For each $k \in [K]$, there exists $i_k \in \{1, \ldots, d \}$ such that
    \begin{equation*}
        B_k = B_0 + \be_{i_k} \bc_k^\top,
    \end{equation*}
    further, $(B_k)^\top \be_{i_k}$ is not a multiple of $(B_0)^\top \be_{i_k}$ unless $i_k$ has no parents in $\cG$.

    \item[(c)] \textbf{Linear observations:} Fix $p \geq d$. There is a full rank matrix $G \in \bbR^{p \times d}$ such that
        $X = G Z$
    in every context $k$.
    Let $H := G^\dagger$ denote its Moore-Penrose pseudoinverse.
    We set the entry of largest absolute value in each row of $H$ to 1.
    If multiple entries in a row have same absolute value we set the leftmost entry to be positive.
\end{itemize}
}
\end{assumption}

Our strongest results hold under one additional assumption.

\begin{assumption}\label{assumption:perfect-interventions}
\textbf{Perfect interventions:} For each $k \in [K]$, there exists $i_k \in \{1, \ldots, d \}$ such that
{
\setlength{\abovedisplayskip}{5pt}
\setlength{\belowdisplayskip}{5pt}
\begin{align*}
    B_k = B_0 + \be_{i_k} \bc_k^\top,
    \quad\quad
\end{align*}
}
where $\bc_k = \lambda_k \be_{i_k} - B_0^\top \be_{i_k}$ for some $\lambda_k > 0$.
\end{assumption}

\begin{remark}[The parts of~\rref{assumption:data-generating} that hold without loss of generality]
Taking $\Var(\varepsilon_i) = 1$ for all $i$ holds without loss of generality, since scaling can be absorbed into the matrix $\Omega_k$.
A linear structural equation model is \emph{causally sufficient} if $\varepsilon_i \independent \varepsilon_j$ for all $i \neq j$.
Thus, for a causally sufficient linear structural equation model, we have $\Cov(\varepsilon) = I_d$ in \rref{assumption:data-generating}(a) without loss of generality.
The ordering of nodes in \rref{assumption:data-generating}(a)
is also without loss of generality: a permutation of the latent nodes can be absorbed into~$G$.
Our ordering makes the matrices $A_k$ upper triangular.

The scaling of $H$ in \rref{assumption:data-generating}(c) is
without loss of generality, as follows.
If $\{ B_k \}_{k=0}^K$ and $H$ satisfy \rref{assumption:data-generating} then $X = (B_k H)^\dagger \varepsilon$. 
Consider the matrices $\{ B_k \Lambda \}_{k=0}^K$ and $\Lambda^\inv H$, for $\Lambda$ diagonal with positive entries.
Observe that $X' = (B_k \Lambda \Lambda^{-1} H)^\dagger \varepsilon$, 
has the same distribution as $X$ in context $k$.
The alternative
matrices satisfy \rref{assumption:data-generating}, except for the scaling condition on $H$.
\rref{assumption:data-generating}(c) therefore
fixes the scaling indeterminacy of each node.

The genericity condition in \rref{assumption:data-generating}(b) automatically holds for perfect interventions.
It fails to hold only for soft interventions that changes the variance but not the edge weights\footnote{i.e., $(\Omega_k)_{i_k,i_k} \neq (\Omega_0)_{i_k,i_k}$ and $(A_k)^\top \be_{i_k} = (A_0)^\top \be_{i_k}$}.
We show the importance of the genericity assumption for the identifiability of causal disentanglement in \rref{appendix:unfaithful-counterexample}.
\end{remark}

\begin{figure}
    \centering
    \includegraphics[width=0.48\textwidth]{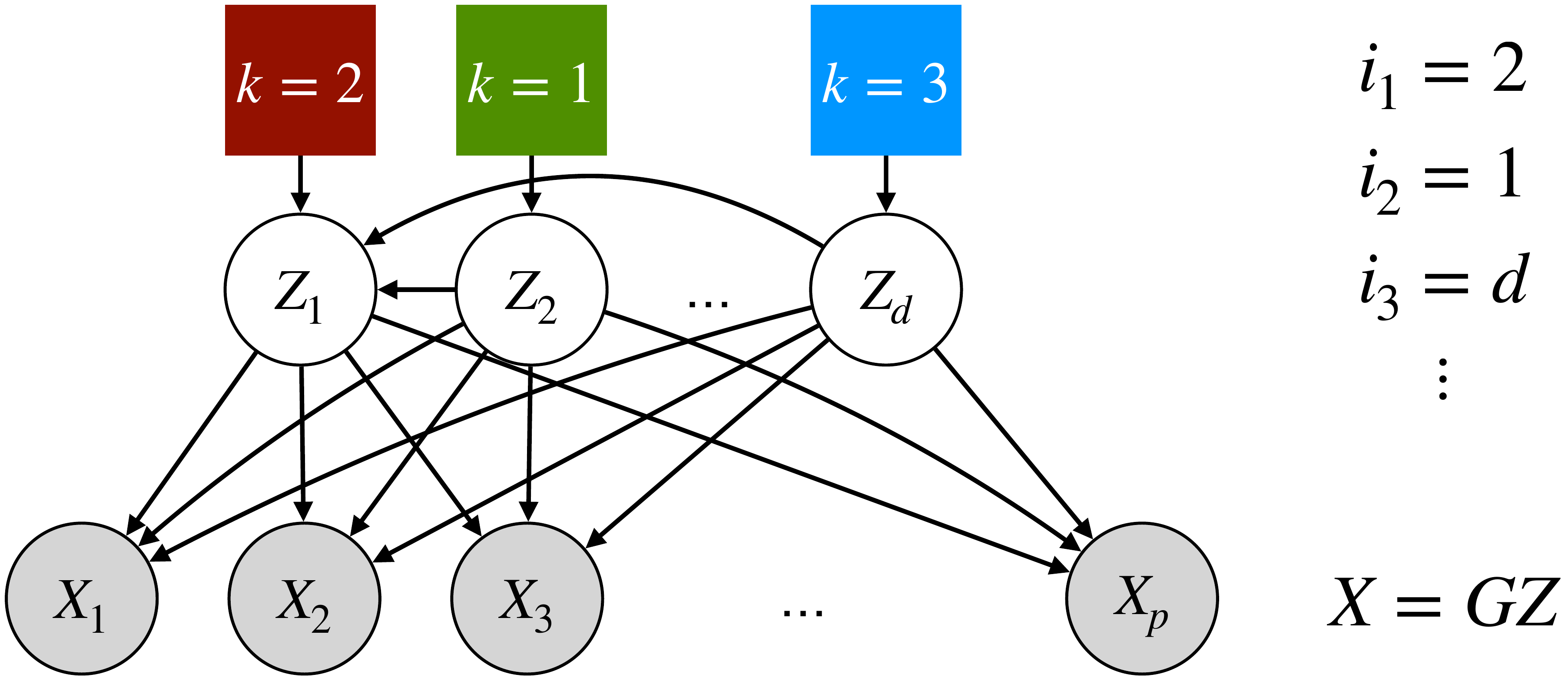}
    \caption{
    \textbf{The proposed setup.}
    The latent variables $Z = (Z_1, \ldots, Z_d)$ follows a linear DAG model, with contexts $k = 1, \ldots, K$ being single node interventions on targets $i_1, \ldots, i_K$.
    The observed variables $X = (X_1, \ldots, X_p)$ are an injective linear function of the latent variables $X = G Z$, where $G \in \bbR^{p \times d}$ does not vary across contexts.
    }
    \label{fig:model}
    \vspace{-.4cm}
\end{figure}

We give an example of a setting in which \rref{assumption:data-generating} might apply.
%
%
Suppose $Z$ is the internal state of a cell (e.g., the concentrations of proteins, the locations of organelles, etc.) and that each context is an exposure to a different small molecule.
\rref{assumption:data-generating}(b) posits that each small molecule has a highly \textit{selective} effect, modifying only one cellular mechanism.
\rref{assumption:perfect-interventions} posits that each small molecule completely disrupts the modified mechanism.
While one does not expect all small molecules to be highly selective, one could filter based on selectivity.

In \rref{appendix:hypothesis-testing}, we describe a hypothesis test to test implications of \rref{assumption:data-generating}(b), and show empirically that this test effectively determines model membership.

The covariance of $X$ in context $k$ is rank deficient when $d < p$, since $X = GZ$.
We therefore define the precision matrix of $X$ in context $k$ to be the pseudoinverse of the covariance matrix, $\Theta_k := \bbE[X X^\top]^\dagger$. Then
{
\begin{equation}\label{eqn:theta-k}
    \Theta_k = H^\top B_k^\top B_k H,
\end{equation}
}
by \rref{prop:pseudoinverse-precision} in \rref{appendix:precision-matrix}, since $\Cov_k(Z) = (B_k^\top B_k)^\inv$. 
%

We consider an unknown latent DAG. Each candidate DAG has unknown weights on its edges, unknown variances on its nodes, unknown new weights under each intervention, and an unknown mixing map to the observed variables.
That is, our goal is to decompose the precision matrices $\{ \Theta_k \}_{k=0}^K$ in \rref{eqn:theta-k} to recover $H$ and $\{ B_k \}_{k=0}^K$.

We recall some graph theoretic notions.
The parents of node $i$ are $\pa_\cG(i) = \{ j \mid j \to i ~\textnormal{in}~\cG \}$, and we define $\Pa_\cG(i) := \pa_\cG(i) \cup \{ i \}$.
Similarly, $\an_\cG(i)$ denotes the \textit{ancestors} of $i$ in $\cG$, the vertices $j$ with a directed path from $j$ to $i$. We define $\An_\cG(i) := \an_\cG(i) \cup \{ i \}$ and $\An_\cG(\cI) := \cup_{i \in \cI}~\An_\cG(i)$.
The source nodes of $\mathcal{G}$ are the nodes $i$ with $\pa_\cG(i) = \varnothing$.
We drop the subscript $\cG$ when the graph is clear from context.

The \textit{transitive closure} of $\cG$, denoted $\barcG$, is the DAG with $\pa_{\barcG}(i) = \an_{\cG}(i)$.
Given a DAG $\cG$, define the partial order $\prec_\cG$ to be $i~\prec_\cG~j$ if and only if $j \in \an_\cG(i)$.
Thus, the transitive closure is the graph with $j \to i$ whenever $i \prec_\cG j$.

To decompose the precision matrices in \rref{eqn:theta-k}, we introduce a matrix decomposition defined on a partial order. 
Recall that the RQ decomposition writes $H \in \bbR^{d \times p}$ as $H= R Q$ for an upper triangular $R \in \bbR^{d \times d}$ and orthogonal $Q \in \bbR^{d \times p}$.
We generalize the RQ decomposition\footnote{The RQ decomposition is used here (rather than the more familiar QR decomposition) because it gives an expression for the \textit{rows} of $H$, which each correspond to one latent variable.}.

\begin{defn}[The partial order RQ decomposition]\label{defn:partial-order-rq-decomposition}
Given a partial order $\prec$, the partial order RQ decomposition writes $H \in \bbR^{d \times p}$ as $H = R Q$, where $R \in \bbR^{d \times d}$ satisfies $R_{ii} \geq 0$ and $R_{ij} = 0$ unless $i \preceq j$, and where $\bq_i$, the $i$-th row of $Q \in \bbR^{d \times p}$, is norm one and orthogonal to $\langle \bq_j : i \prec j \rangle$.
\end{defn}
We recover the usual reduced RQ decomposition~\citep{trefethen1997numerical} when $\prec$ is the total order $1 < \cdots < d$.
We construct the partial order RQ decomposition in \rref{appendix:partial-order-rq}.
Finally, given a positive definite matrix $M \in \bbR^{d \times d}$, the \textit{Cholesky factor} $U \in 
\bbR^{d 
\times d}$, denoted $\cholesky(M)$, is the unique upper triangular matrix with positive diagonal such that $M = U^\top U$.

\section{Identifiability of Causal Disentanglement}

We establish the sufficiency and worst case necessity of one intervention per latent node for identifiability of our causal disentanglement problem. 
The following result describes recovery of $\cG$. Later, we discuss identifiability of the parameters in our setup.
Since the labeling of latent nodes is unimportant, $\cG$ is recovered if it is found up to relabeling.

\begin{thm}
\label{thm:G_and_Gbar}
Assume the setup in Assumption~\ref{assumption:data-generating} with  $d$ latent variables. Then $d$ interventions are sufficient and, in the worst case, necessary to recover $\barcG$ from $\{ \Theta_k \}_{k \in K}$.
If \rref{assumption:perfect-interventions} also holds, then $d$ interventions are sufficient and, in the worst case, necessary to recover $\cG$ from $\{ \Theta_k \}_{k \in K}$.
\end{thm}

\subsection{Preliminaries}\label{sec:reduction}
We note the following basic fact, where $\bv^{\otimes 2} := \bv \bv^\top$:
\begin{fact}\label{fact:rank-one-decomposition}
    Let $B \in \bbR^{d \times d}$.
    Then $B^\top \! B = \sum_{i=1}^d (B^\top \be_i)^{\otimes 2}$.
\end{fact}
We give a proof in \rref{appendix:reduction}.
This fact gives the key identity that drives our identifiability results.

\begin{prop}\label{prop:key-identity}
    Consider the setup in \rref{assumption:data-generating}.
    Then, for any $k \in [K]$,
    \begin{equation*}
        \Theta_k - \Theta_0 = 
        (H^\top B_k^\top \be_{i_k})^{\otimes 2} - (H^\top B_0^\top \be_{i_k})^{\otimes 2}.
    \end{equation*}
    In particular the rank of the difference $\Theta_k - \Theta_0$ is $1$ if $i_k$ is a source node in $\mathcal{G}$, and $2$ otherwise.
\end{prop}
\begin{proof}
By \rref{assumption:data-generating}, $B_k^\top \be_i = B_0^\top \be_i$ for all $i \neq i_k$.
Using \rref{fact:rank-one-decomposition}, we have 
\[
    B_k^\top B_k - B_0^\top B_0 
    = 
    (B_k^\top \be_{i_k})^{\otimes 2} 
    - 
    (B_0^\top \be_{i_k})^{\otimes 2}.
\]
Recall from \rref{eqn:theta-k} that $\Theta_k = H^\top B_k^\top B_k H$.
The result follows from left-multiplying both sides of the displayed equation by $H^\top$ and right-multiplying by $H$.
This shows that ${\rm rank}(\Theta_k - \Theta_0) \leq 2$. For a source node, both vectors $B_k^\top \be_{i_k}$ and $B_0^\top \be_{i_k}$ have just one entry non-zero and ${\rm rank}(\Theta_k - \Theta_0) = 1$. Otherwise, the vectors have more than one entry non-zero and, by the genericity condition in \rref{assumption:data-generating}(b), the difference $\Theta_k - \Theta_0$ has rank two.
\end{proof}

We can reduce a more general causal disentanglement problem to our setting, as we explain in \rref{appendix:reduction}.
First, we can count the latent dimension, since $\rank(\Theta_k) = d$ for any~$k$.
Second, we can identify which environments correspond to interventions on the same intervention target, see \rref{prop:identify-same-target}.
Finally, we can identify which environment is observational using rank constraints, see \rref{prop:identify-observational}.
Thus, we assume without loss of generality that $d$ is known, that the observational environment is known, and that each node is only intervened on in one context.

\subsection{Sufficiency}
\label{sec:theoretical-sufficiency}

We define $S(\cG)$ to be the set of permutations on $d$ letters such that $\sigma(j) > \sigma(i)$ for all edges $j \to i$.
For example, if $\cG$ is a complete graph then $S(\cG)$ contains only the identity.
If $\cG$ has no edges then $S(\cG)$ is the group of permutations on $d$ letters. 
The permutation matrix corresponding to permutation $\sigma$ is $P_\sigma \in \bbR^{d \times d}$ with $(P_\sigma)_{ij} = \kron_{\{ i = \sigma(j) \}}$. 
Our main sufficiency result is the following.

\begin{thm}
\label{thm:main_id_non_constructive}
Assume the set-up in Assumptions \ref{assumption:data-generating} and \ref{assumption:perfect-interventions} with one intervention on each latent node. Then the graph $\cG$, the intervention targets $i_k$, and the parameters are identifiable up to $S(\cG)$: given a solution $(B_0, \ldots, B_K, H)$, the set of solutions is $\{ ( P_\sigma B_0 P_\sigma\T, \ldots,
P_\sigma B_K P_\sigma\T, P_\sigma H) : \sigma 
\in S(\cG) \}$. 
\end{thm}

Theorem~\ref{thm:main_id_non_constructive} says that solutions to the causal disentanglement problem are unique up to permutations of the latent nodes that preserve the property that $j \to i$ implies $j > i$.
First, we verify that each permutation in $S(\cG)$ gives a solution. 

\begin{prop}
\label{prop:sigma_solutions}
Assume the set-up in \rref{assumption:data-generating}.
Given a solution
 $(B_0, \ldots, B_K, H)$ to \rref{eqn:theta-k} for $k \in \{0\} \cup [K]$,
the matrices 
 $( P_\sigma B_0 P_\sigma\T, \ldots,
P_\sigma B_K P_\sigma\T, P_\sigma H)$
are a valid solution whenever $\sigma 
\in S(\cG)$. 
\end{prop}

\begin{proof}
Let $\{ B_k \}_{k=0}^K$ and $H$ satisfy \rref{assumption:data-generating} and \rref{eqn:theta-k}.
Define $\Bsigma_k = P_\sigma B_k P_\sigma^\top$ and $\Hsigma = P_\sigma H$ for $\sigma \in S(\cG)$.
Then $\Theta_k = {\Hsigma}^\top \! {\Bsigma_k}^\top \! \Bsigma_k \Hsigma$.
The matrices $\Bsigma_k$ are upper triangular, as follows.
For all $i, j \in [p]$, we have $(\Bsigma_k)_{\sigma(i), \sigma(j)} = (B_k)_{ij}$. 
%
Hence $\Bsigma_k$ is upper triangular when $(B_k)_{ij} = 0$ for all $i, j \in [p]$ with $\sigma(i) > \sigma(j)$. 
This holds since $\sigma \in S(\cG)$.
%
%
Moreover, these matrices also satisfy \rref{assumption:data-generating}(b) with the intervention target $\sigma(i_k)$ in context $k$.
Finally, $\Hsigma$ satisfies \rref{assumption:data-generating}(c), since we just permute the rows of $H$.
\end{proof}

\begin{algorithm}[t]
	\caption{\IdentifyAncestors}\label{algm:identify-ancestors}
	\begin{algorithmic}[1]
		\STATE \textbf{Input:} $\Theta_k$, $\Theta_0$, $\{ \hatbq_i \}_{i \in \cI}$
		\STATE \textbf{Output:} Vector $\hatbq_k$, ancestor set $\cA$
		\STATE Let $\cA = \cI$
		\FOR{$i \in \cI$}
		    \STATE Let $W_{\neg i} = \langle \hatbq_i: j \in \cI \setminus \{ i \} \rangle$
		    \STATE Let $V_{\neg i} = \proj_{W_i^\perp} \rowspan(\Theta_k - \Theta_0)$
    		\STATE If $\dim(V_{\neg i}) = 1$, let $\cA = \cA \setminus \{ i \}$ \label{line:dimension-test-ancestors}
		\ENDFOR
		\STATE Let $W = \langle \hatbq_a : a \in \cA \rangle$
		\STATE Let $V = \proj_{W^\perp} \rowspan(\Theta_k - \Theta_0)$
		\STATE Take $\hatbq_k$ with first nonzero entry positive and $\| \hatbq_k \|_2 = 1$, such that $\langle \hatbq_k \rangle = V$
		\STATE \textbf{return} $\hatbq_k$, $\cA$
	\end{algorithmic}
\end{algorithm}

\begin{algorithm}[t]
	\caption{\IdentifyPartialOrder}\label{algm:identify-partial-order}
	\begin{algorithmic}[1]
		\STATE \textbf{Input:} Precision matrices $(\Theta_0, \Theta_1, \ldots, \Theta_K)$, rank $d$
		\STATE \textbf{Output:} Factor $\hatQ$, partial order $\prec$
		\STATE Let $\cI_0 = \{ \}$, $\hatQ = \bzero_{d \times d}$
		\FOR{$t = 1, \ldots K$}
		    \STATE Let $W_t = \langle \hatbq_i : i \in \cI_{t-1} \rangle$
		    \STATE Let $V_k = \proj_{W_t^\perp} \rowspan(\Theta_k - \Theta_0)$ for $k \not\in \cI_{t-1}$
    		\STATE Pick $k$ such that $\dim(V_k) = 1$\label{line:dimension-test-partial-order}
    		\STATE Let $\hatbq_k, \cA = \IdentifyAncestors(\Theta_k, \Theta_0, \{ \hatbq_i \}_{i \in \cI_{t-1}})$\label{line:pick-next-row-q}
    		\STATE Add $a' \succ k$ for any $a' \succeq a$, $a \in \cA$
    		\label{line:add-ancestors}
    		\STATE Let $\cI_t = \cI_{t-1} \cup \{ k \}$, $\hatQ_t = [\hatbq_k; \hatQ_{t-1}]$
		\ENDFOR
		\STATE \textbf{return} $\hatQ$, $\prec$
	\end{algorithmic}
\end{algorithm}

We give a constructive proof of \rref{thm:main_id_non_constructive} via an algorithm to recover $H$ and $\{ B_k \}_{k=0}^K$ from $\{ \Theta_k \}_{k=0}^K$\footnote{
We only use the second moment of $X$.
We do not use the first moment since we assume $\bbE[\varepsilon] = 0$, and we do not use higher moments since in the worst case (Gaussian noise), they contain no additional information.
}.
The computational complexity of the algorithm is given in \rref{appendix:computational-complexity}.
The bulk of the algorithm is devoted to recovering $H$.
First, we recover the partial order $\prec_\cG$
(i.e., the DAG $\barcG$), together with the matrix $Q$ from a partial order RQ decomposition of $H$, up to signs and permutations of rows, in \rref{algm:identify-partial-order}.
The subroutine \rref{algm:identify-ancestors} recovers the ancestors of a node and its corresponding row of $Q$.
Then we recover $R$ in \rref{algm:iterative-difference-projection}. 
Having recovered $H$, the matrices $\{ B_k \}_{k=0}^K$ are found via the Cholesky decomposition.

%
%

We show that \rref{algm:identify-partial-order} returns $Q$ from the partial order RQ decomposition of $H$, up to a permutation $\sigma \in S(\cG)$ and a matrix in $\Sig_d$, the $d \times d$ diagonal matrices with diagonal entries $\pm 1$.

\begin{prop}
\label{prop:orthogonal-correctness}
Assume the setup in \rref{assumption:data-generating}, and that every latent node is intervened on; i.e, $\{ i_k \}_{k=1}^K = [d]$.
Let $(\hatQ, \prec)$ be the output of \rref{algm:identify-partial-order}.
Then $\prec$ is the partial order $\prec_\cG$.
Moreover, let $H = R Q$ be the partial order RQ decomposition of $H$ for the partial order $\prec_\cG$.
Then $\hatQ = S P_\sigma Q$ for some $\sigma \in S(\cG)$ and $S \in \Sig_d$.
\end{prop}

The following lemma relates the partial order of $\cG$ to the linear spaces $\rowspan(\Theta_k - \Theta_0)$. 

\begin{lemma}\label{lemma:rowspan-inclusion}
Assume the setup in \rref{assumption:data-generating}.
Let $H = R Q$ be a partial order RQ decomposition of $H$.
Let $i_k$ be the intervention target of context $k$ and let $\cI \subseteq [d]$.
Then
\vspace{-0.3cm}
\begin{enumerate}
    \item[(a)] $\rowspan(\Theta_k - \Theta_0) 
    \subseteq 
    \langle \bh_i : i \in \cI \rangle$
    if and only if $\Pa(i_k) \subseteq \cI$,
    
    \item[(b)] $\rowspan(\Theta_k - \Theta_0) \subseteq \langle \bq_i : i \in \An(i_k) \rangle$, and
    
    \item[(c)] $\rowspan(\Theta_k - \Theta_0) \not\subseteq \langle \bq_i : i \in \cI \rangle$ if $\Pa(i_k) \not\subseteq \cI$.
\end{enumerate}
\end{lemma}

\begin{proof}
Let $\bh_i$ denote the $i$th row of $H$, and let $\boldalpha_{j,i_k} := \sum_{i \in \Pa(i_k)} (B_j)_{i_k,i} \bh_i$.
By \rref{prop:key-identity}, we have
\begin{equation}\label{eqn:rowspan-difference}
    \rowspan(\Theta_k - \Theta_0) = \left\langle \boldalpha_{k,i_k}, \boldalpha_{0,i_k} \right\rangle.
\end{equation}

\rref{eqn:rowspan-difference} shows that $\rowspan(\Theta_k - \Theta_0) \subset \langle \bh_i : i \in \Pa(i_k) \rangle$. This linear space is contained in $\langle \bh_i : i \in \cI \rangle$ whenever $\Pa(i_k) \subseteq \cI$. 
Conversely, assume there exists $j \in \Pa(i_k)$ with $j \notin \cI$. Then, containment of $\rowspan(\Theta_k - \Theta_0)$ in $\langle \bh_i : i \in \cI \rangle$ cannot hold: containment implies $(B_0)_{i_k, j} \bh_j \in \langle \bh_i : i \in [p] \backslash \{ j \} \rangle$, a contradiction since $H$ is full row rank and $(B_0)_{i_k, j} \neq 0$.
Hence (a) holds.

By definition of the partial order RQ decomposition, we have $\bh_i \in \langle \bq_j \mid j \in \An(i) \rangle$.
Thus, by transitivity of the ancestorship relation, (b) holds.
Conversely, assume there exists $j \in \Pa(i_k)$ with $j \notin \cI$.
Then $\rowspan(\Theta_k - \Theta_0) 
    \subseteq 
    \langle \bq_i : i \in \cI \rangle$ implies that 
    $$ \bh_j 
    \in 
    \langle \bq_i : i \in \cI \rangle + \langle \bh_i : i \in \Pa(i_k) \setminus \{ j \} \rangle,$$
since $(B_0)_{i_k,j} \neq 0$.
We partition $\cI$ into the descendants and non-descendants of $j$: let $\cI_d := \{ i \in \cI : j \in \an(i) \}$ and let $\cI_{nd} := \{ i \in \cI : j \notin \an(i) \}$.
By definition of the partial order RQ decomposition, we have $\bq_i \perp \bh_j$ whenever $j \in \an(i)$.
Thus $\bh_j  
    \in
    \langle \bq_i : i \in \cI_{nd} \rangle + \langle \bh_i : i \in \Pa(i_k) \setminus \{ j \} \rangle$.
Inverting the partial order RQ decomposition gives $\bq_i \in \langle \bh_{i'} : i' \in \An(i) \rangle$.
Hence
\begin{align*}
    \bh_j  
    &\in
    \langle \bh_i : i \in \An(\cI_{nd}) \cup ( \Pa(i_k) \setminus \{ j \}) \rangle,
\end{align*}
a contradiction, since $j \notin \An(\cI_{nd})$ and $H$ is full rank.
\end{proof}

\begin{proof}[Proof of \rref{prop:orthogonal-correctness}]
Assume that $\hatQ_{t-1}$ is the last $t-1$ rows of $S P_\sigma Q$, for some $\sigma \in S(\cG)$ and $S \in \Sig_d$.
%
%
Then $W_{t-1}= \langle \hatbq_i : i \in \cI_{t-1} \rangle.$
At step $t$, we pick $k$ such that $V_k = \proj_{W_t^\perp} \rowspan(\Theta_k - \Theta_0)$ has dimension one.
\rref{lemma:rowspan-inclusion} implies that such $k$ are those with $\pa(i_k) \subseteq \cI_{t-1}$.
\rref{algm:identify-ancestors} returns a set $\cA$ with $\pa(i_k) \subseteq \cA \subseteq \an_\cG(i_k)$. 
Thus \rref{line:add-ancestors} adds the relation $k \prec a'$ if and only if 
$a' \in \an_\cG(i_k)$. Hence $\prec$ is the partial order $\prec_\cG$.
\rref{line:pick-next-row-q} picks $\hatbq_k$ orthogonal to $\{ \bq_i : i \succ_\cG i_k \}$, such that $\rowspan(\hatQ_t)$ contains $\rowspan(\Theta_k - \Theta_0)$.
By \rref{lemma:rowspan-inclusion} and \rref{defn:partial-order-rq-decomposition}, this is $\pm \bq_{i_k}$.
Thus, $\hatQ_t$ is equal to the last $t$ rows of $S' P_{\sigma'} Q$, for some $\sigma' \in S(\cG)$ and $S' \in \Sig_d$. 
Repeating for $t = 1, \ldots, K$ gives the result.
\end{proof}

\begin{algorithm}[t]
	\caption{\IterativeDifferenceProjection}\label{algm:iterative-difference-projection}
	\begin{algorithmic}[1]
		\STATE \textbf{Input:} Precision matrices $(\Theta_0, \Theta_1, \ldots, \Theta_K)$
		\STATE \textbf{Output:} $\hatH$, $(\hatB_0, \hatB_1, \ldots, \hatB_K )$
		\STATE Let $d = \rank(\Theta_0)$
		\STATE Let $\hatQ, \prec = \IdentifyPartialOrder( (\Theta_0, \Theta_1, \ldots, \Theta_K), d)$
		\STATE Let $\hatC_k = \cholesky((\hatQ^{\dagger})^\top \Theta_k \hatQ^\dagger)$ for $k = 0, \ldots, K$
		\STATE Let $\hatR = I_d$
		\FOR{$k = 1, \ldots, K$}
    		\STATE Let $\hatD_k = \hatC_k - \hatC_0$
    		\STATE Let $\hati_k$ be index of the only nonzero row of $\hatD_k$ \label{line:pick-nonzero-row}
    		\STATE Let $\hatR_{\hati_k} = (\hatD_k)_{\hati_k} + (\hatC_0)_{\hati_k}$
		\ENDFOR
		\STATE Let $\hatH' = \hatR \hatQ$
		\STATE Let $\hatH = \hLambda \hatH'$, for $\hLambda$ diagonal such that $\hatH$ satisfies the conditions on $H$ in \rref{assumption:data-generating}(c)
		\label{line:scale-H}
		\STATE Let $\hatB_0 = \cholesky( (\hatH^\dagger)^\top \Theta_0 \hatH^\dagger)$ 
		\STATE Let $\hatB_k = \hatB_0 + \be_{\hati_k} \left( |\hLambda_{\hati_k,\hati_k}| \be_{\hati_k} - \hatB_0^\top \be_{\hati_k} \right)^\top$ for $k = 1, \ldots, K$
		\label{line:compute-bk}
		\STATE \textbf{return} $\hatH$, $( \hatB_0, \hatB_1, \ldots, \hatB_K )$
	\end{algorithmic}
\end{algorithm}

We prove \rref{thm:main_id_non_constructive} by proving the following result, which is its constructive analogue.

\begin{thm}\label{thm:main-identifiability}
Assume the setup in Assumptions \ref{assumption:data-generating} and \ref{assumption:perfect-interventions}, and that every latent node is intervened; i.e., $\{ i_k \}_{k=1}^K = [d]$.
Let $\hatH$ and $\{ \hatB_k \}_{k=0}^K$ be the output of \rref{algm:iterative-difference-projection}.
Then $\hatH = P_\sigma H$ and $\hatB_k = P_\sigma B_k P_\sigma^\top$ for all $k$, for some $\sigma \in S(\cG)$.
\end{thm}

\begin{proof}
Let $H = RQ$.
Then
$\Theta_k = Q\T \! R\T \! B_k\T \! B_k R Q$, by \rref{eqn:theta-k}, and $\hatQ = S P_\sigma Q$ for some $\sigma \in S(\cG)$ and $S \in \Sig_d$, by \rref{prop:orthogonal-correctness}. Hence $(\hatQ^\dagger)^\top \Theta_k \hatQ^\dagger$ equals 
{
\setlength{\abovedisplayskip}{5pt}
\setlength{\belowdisplayskip}{5pt}
\begin{multline*}\label{eqn:orthogonalized}
    S \left(P_\sigma R^\top \! P_\sigma^\top \right)
    \left(P_\sigma B_k^\top P_\sigma^\top \right)
    \left( P_\sigma B_k P_\sigma^\top \right)
    \left( P_\sigma R P_\sigma^\top \right) S.
\end{multline*}
}
Let $\Csigma_k = S (P_\sigma B_k P_\sigma^\top) (P_\sigma R P_\sigma^\top) S$.
The matrix $\Csigma_k$ is upper triangular, since it is a product of four upper triangular matrices, by the definition of $S(\cG)$ and of the partial order RQ decomposition, where we use that $i \prec_\cG j$ implies $i < j$.
Moreover, $\Csigma_k$ has positive diagonal, since the matrices $R$ and $B_k$ have positive diagonal, by  \rref{defn:partial-order-rq-decomposition} and \rref{assumption:data-generating}(a) respectively.
Hence $\hatC_k := \Csigma_k$ is the Cholesky factor of $(\hatQ^\dagger)^\top \Theta_k \hatQ^\dagger$.
The differences $\hatD_k := \hatC_k - \hatC_0$ equal $S P_\sigma (B_k - B_0) R P_\sigma^\top S = S P_\sigma \be_{i_k} \bc_k^\top R P_\sigma^\top S$, by \rref{assumption:data-generating}(b). 
The intervention target $\sigma(i_k)$ is the only nonzero row of $\hatD_k$, i.e., $\hati_k = \sigma(i_k)$.
Observe that $(\hatD_k)_{\hati_k} + (\hatC_0)_{\hati_k} = S_{\sigma(i_k),\sigma(i_k)} \lambda_k (R P_\sigma^\top)_{\sigma(i_k)}$.
%
%
Thus, we have recovered $\hatR = \hLambda P_\sigma R P_\sigma^\top$ for $\hLambda$ diagonal such that $\hLambda_{\hati_k,\hati_k} = \pm \lambda_k$.
This gives $\hatH' = \hLambda P_\sigma R P_\sigma^\top P_\sigma Q = \hLambda P_\sigma H$.
The scaling in \rref{line:scale-H} recovers $\hatH = P_\sigma H$ and $\hLambda$.
We have $(\hatH^\dagger)^\top \Theta_0 \hatH^\dagger = P_\sigma B_0^\top P_\sigma^\top P_\sigma B_0 P_\sigma^\top$, where $P_\sigma B_0 P_\sigma^\top$ is upper triangular, and thus we recover $\hatB_0 = P_\sigma B_0 P_\sigma^\top$ from the Cholesky decomposition.
Finally, since $|\hLambda_{\hati_k,\hati_k}| = \lambda_k$, \rref{line:compute-bk} gives us $\hatB_k = P_\sigma B_k P_\sigma^T$.
\end{proof}

\rref{thm:main-identifiability} requires \rref{assumption:perfect-interventions}, see \rref{appendix:soft-intervention-counterexample}.
In \rref{appendix:weight-variant-comparison}, we compare our identifiability condition to that of \citet{liu2022weight}.
We show that \citet{liu2022weight} requires that the latent graph has fewer than $d$ edges.
In contrast, our condition imposes no constraints on the latent graph.

\begin{figure*}[t]
    \centering
    \begin{subfigure}{.3\textwidth}
        \includegraphics[width=\textwidth]{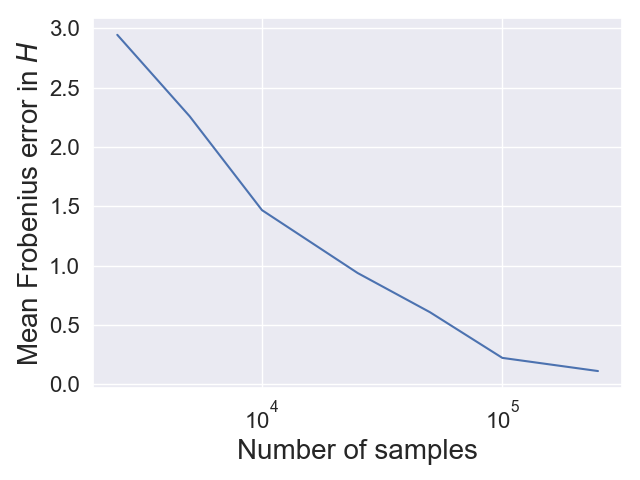}
        \caption{Error in estimating $H$}
    \end{subfigure}
    ~
    \begin{subfigure}{.3\textwidth}
        \includegraphics[width=\textwidth]{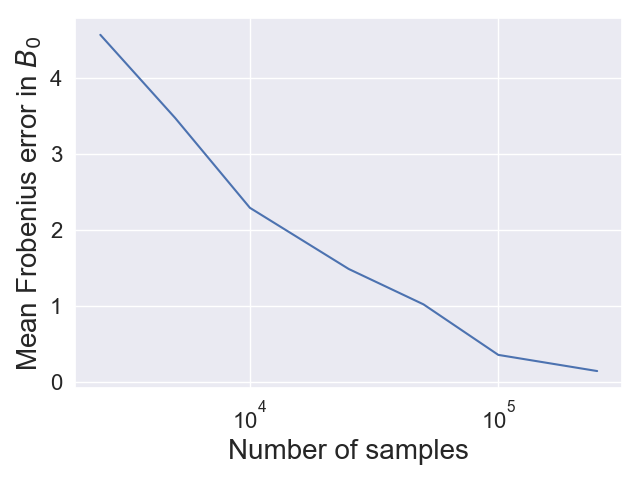}
        \caption{Error in estimating $B_0$}
    \end{subfigure}
    ~
    \begin{subfigure}{.297\textwidth}
        \includegraphics[width=\textwidth]{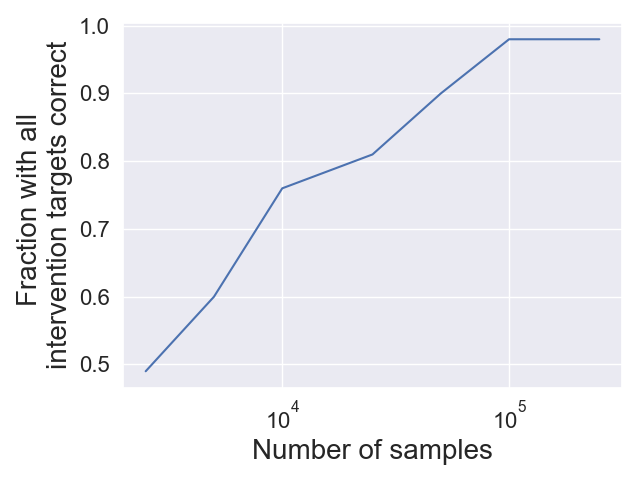}
        \caption{Intervention targets}
    \end{subfigure}
    \caption{
    \textbf{The adapted version of \rref{algm:iterative-difference-projection} is consistent for recovering $H$, $B_0$, and $\{ i_k \}_{k=1}^K$ from noisy data.}
    At each sample size, we generate 500 random models.
    Note the logarithmic scale on the x-axis.
    In \textbf{(a)}, we plot the median of $\| \hatH - H \|_2$, the error in Frobenius norm.
    In \textbf{(b)}, we plot the median of $\| \hatB_0 - \hatB \|_2$.
    In \textbf{(c)}, we plot the fraction of models where all intervention targets were correctly estimated.
    }
    \label{fig:results}
\end{figure*}

\subsection{Worst-case Necessity}\label{sec:theoretical-necessity}

We show that one intervention per latent node is \textit{necessary} for identifiability of our setup, in the worst case.
It is clear that observational data does not suffice for identifiability: from just observational data, we always have a solution with \textit{independent} latent variables, as follows.
Let $\hatH = \Lambda B_0 H$ and $\hatB_0 = \Lambda^\inv$, for $\Lambda$ a diagonal matrix with positive entries such that \rref{assumption:data-generating}(c) holds.
Then $\hatB_0 \hatH = B_0 H$, i.e. $\hatB_0$, $\hatH$ solve the causal disentanglement problem.
The new solution has independent latent nodes, since $\hatB_0$ is diagonal.

The next result, which follows from prior work in causal structure learning, says that $d-1$ interventions are required in the worst case, for a fully observed model.
This is the special case of our setup where $H$ is a permutation matrix.

\begin{prop}\label{prop:worst-case-necessity-observed}
Assume the setup in Assumptions \ref{assumption:data-generating} and \ref{assumption:perfect-interventions}, with $H$ a permutation matrix.
Let $K < d - 1$ with all intervention targets distinct.
Then, in the worst case over intervention targets $\{ i_k \}_{k=1}^K$, $B_0$ and $H$ are not identifiable.
\end{prop}
\vspace{-.25cm}
\begin{proof}
In the linear Gaussian setting with unknown-target interventions, the structure of a DAG is only identifiable up to its \textit{interventional Markov equivalence class} (MEC), see e.g. Proposition 3.3(ii) of \citet{castelletti2022network}.
For single-node interventions, $d - 1$ interventions are in the worst case necessary to ensure that the interventional MEC has size one, by Theorem 3.7 of \citet{eberhardt2005number}.
\end{proof}

\vspace{-.2cm}
We show that $d$ interventions are necessary, in the worst case, when $H$ is a general matrix. The proof is in \rref{appendix:proofs}.

\begin{prop}\label{prop:non-identifiability}
Assume the setup in Assumptions \ref{assumption:data-generating} and \ref{assumption:perfect-interventions}, with $K < d$.
Then there exist $H$ and $\{ B_k \}_{k=0}^K$ and a set of $K$ distinct intervention targets such that (i) $H$ and $B_k$ are not identifiable up to $S(\cG)$ and (ii) $\prec_\cG$ is not identifiable.
\end{prop}

\vspace{.1cm}
\begin{example}
Proposition~\ref{prop:non-identifiability} generalizes our motivating example from \rref{sec:intro}.
Fix $H \in \bbR^{2 \times 2}$ with entries $H_{ij}$, and fix upper triangular $B_0, B_1 \in \bbR^{2 \times 2}$ with entries $(B_0)_{ij}$ and $(B_1)_{ij}$, respectively.
Assume $i_1 = 2$; i.e., $(B_0)_{11} = (B_1)_{11}$ and $(B_0)_{12} = (B_1)_{12}$.
Let
\begin{small}
\begin{align*}
    \hatB_0 \! &= \! \begin{bmatrix}
    1 & 0
    \\
    0 & (B_0)_{22}
    \end{bmatrix},
    \quad
    \hatB_1 \! = \! \begin{bmatrix}
    1 & 0
    \\
    0 & (B_1)_{22}
    \end{bmatrix},
    \\ 
    \hatH \! &= \! \begin{bmatrix}
    (B_0)_{11} H_{11} + (B_0)_{12} H_{21} 
    & 
    (B_0)_{11} H_{12} + (B_0)_{12} H_{22} 
    \\
    H_{21} & H_{22}
    \end{bmatrix} \!.
\end{align*}
\end{small}
\vspace{-.5cm}

Then for $k \in \{ 0, 1\}$, we have $\hatB_k \hatH = B_k H$, both equal to
\begin{multline*}
    \begin{bmatrix}
    (B_0)_{11} H_{11} + (B_0)_{12} H_{21}
    &
    (B_0)_{11} H_{12} + (B_0)_{12} H_{22} 
    \\
    (B_k)_{22} H_{21}
    &
    (B_k)_{22} H_{22}
    \end{bmatrix}
\end{multline*}
\end{example}

\begin{remark}
\rref{prop:non-identifiability} pertains to the worst case over intervention targets. 
It is natural to ask if there exists a better choice of $K$ intervention targets, for $K < d$, such that $H$ and $\{ B_k \}_{k=0}^K$ are identifiable.
For example, when $d = 2$, consider $\cG = \{ 2 \to 1 \}$, with an intervention on $Z_1$; i.e., $i_1 = 1$.
Then $\rowspan(\Theta_1 - \Theta_0) \subseteq \langle \bh_i : i \in \cI \rangle$ if and only if $\cI = \{ 1, 2 \}$, by \rref{lemma:rowspan-inclusion}(a).
Thus, $\Theta_1 - \Theta_0$ is rank 2, and we can detect that $i_1 = 1$ and $\cG = \{ 2 \to 1 \}$; otherwise, we would have $\rank(\Theta_1 - \Theta_0) = 1$.
While $i_1$ and $\cG$ are identifiable, preliminary computational evidence suggests that the entries of $B_0, B_1$, and $H$ are not identifiable.
\end{remark}

\begin{proof}[Proof of~\rref{thm:G_and_Gbar}]
The necessity of $d$ interventions is \rref{prop:non-identifiability}.
Under Assumption~\ref{assumption:data-generating} and~\ref{assumption:perfect-interventions}, the sufficiency of $d$ interventions follows from \rref{thm:main_id_non_constructive}. Under \rref{assumption:data-generating}, the sufficiency is the recovery of $\prec_\cG$ in \rref{prop:orthogonal-correctness}.
\end{proof}

\section{Experimental Results}\label{sec:experimental}

We adapt our proof of \rref{thm:main-identifiability} into a method for causal disentanglement in the finite-sample setting. 
We modify our methods to (1) use matrices instead of vector spaces, (2) use scores based on singular values to test rank, and (3) choose a nonzero row based on norms.
The adapted algorithms are in \rref{appendix:noisy-algorithm}.
In this section, we investigate the performance of the method in a simulation study.
There is a single hyperparameter $\gamma \in [0, 1]$, the percentage of spectral energy associated to the largest singular value, above which we consider a matrix to have rank one.
We use $\gamma = 0.99$.

\subsection{Synthetic Data Generation}

We generate 500 random models following \rref{assumption:data-generating} for $d = 5$ latent and $p = 10$ observed variables, as follows.
We sample the graph $\cG$ from an Erd\H{o}s-R{\'e}nyi random graph model with density $0.75$.
We sample the nonzero entries of $A_0$ independently from $\Unif(\pm [0.25, 1])$, and the nonzero entries of $\Omega_0$ independently from $\Unif([2, 4])$.
We sample uniformly among permutations to generate the intervention targets $i_k$.
In context $k$, we have $A_k = A_0 - \be_{i_k} A_0^\top \be_{i_k}$; i.e., all entries in row $i_k$ are 0.
We change $(\Omega_0)_{i_k,i_k}$ into a new value $(\Omega_k)_{i_k,i_k}$, sampled from $\Unif([6, 8])$ to ensure a non-negligible change.
Finally, the entries of $H$ are sampled independently from $\Unif([-2, 2])$.
We consider sample sizes $n$ from $2500$ to $250000$ and use as input the sample precision matrices.
All code for data generation and for our adapted versions of Algorithms \ref{algm:identify-ancestors}, \ref{algm:identify-partial-order}, and \ref{algm:iterative-difference-projection} (that is, Algorithms \ref{algm:identify-ancestors-noisy}, \ref{algm:identify-partial-order-noisy} and \ref{algm:iterative-difference-projection-noisy}) can be found at the link in \rref{appendix:code-and-data}.

\subsection{Synthetic Data Results}

We show the results of applying our method in the finite-sample setting, presented in \rref{fig:results}.
We measure the error in estimating the parameters $H$ and $B_0$ and the intervention targets $\{ i_k \}_{k=1}^K$.
Since the problem is only identifiable up to the partial order $\prec_\cG$, we align our estimated $\hatH$, $\hatB_0$, and $\{ \hati_k \}_{k=1}^K$ with $H$, $B_0$, and $\{ i_k \}_{k=1}^K$ by picking $\sigma \in S(\cG)$ to maximize $\sum_{k=1}^K \kron\{i_k = \sigma(\hati_k)\}$.
For small $d$, this optimization can be solved by enumerating over $S(\cG)$.
For large $d$, we use the integer linear program in \rref{appendix:best-matching}, implemented in \texttt{gurobipy}.
As expected, the estimation errors for $H$ and $B_0$ decrease as the sample size increases, while the fraction of models with all intervention targets correctly estimated increases.

\subsection{Biological Data Results}

\begin{figure*}[hbtp]
    \centering
    \includegraphics[width=\textwidth]{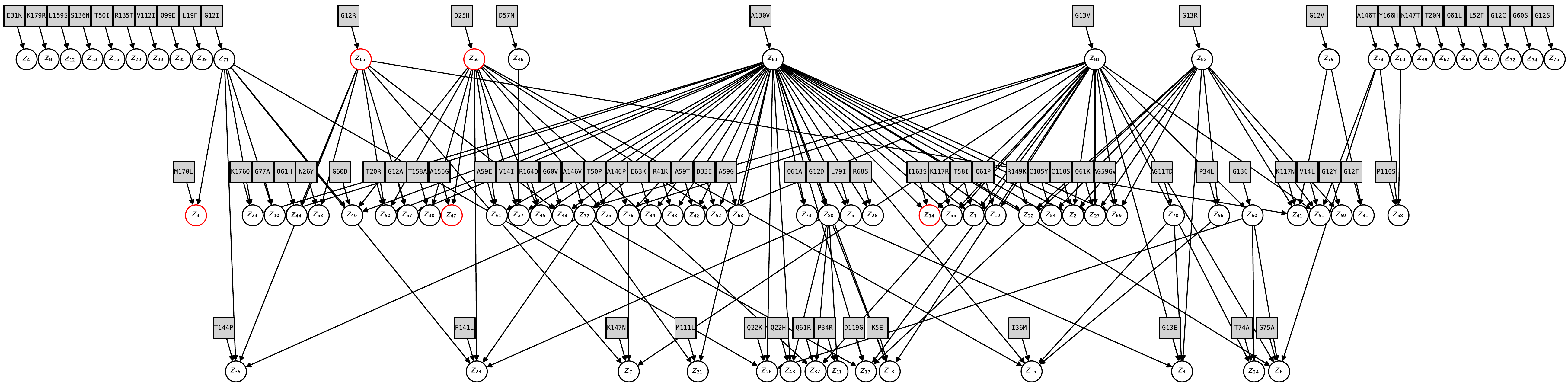}
    \caption{
    \textbf{The latent graph and intervention targets learned from scRNA-seq data.} 
Edges with weight of magnitude above 0.2 are shown.
    Boxes represent context indicators, corresponding to KRAS mutations, with edges to their respective targets.
    Red nodes are significantly associated with survival outcomes in the TCGA dataset.
    }
    \label{fig:latent-graph-real-data}
\end{figure*}

We evaluate our method
on a 
dataset from \citet{ursu2022massively}.
This single-cell RNA sequencing (scRNA-seq) dataset consists of 90,000 cells from a lung cancer cell line, with 83 different nonsynonymous mutations of the KRAS oncogene overexpressed.

\textbf{Semi-synthetic analysis.}
The authors divide the mutations into five categories based on the genes that they affect, and compute a score for the impact of each mutation.
Taking the two highest impact mutations from each category gives $K = 10$ contexts.
The wild type KRAS gene is taken as the observational context.
We select $p = 100$ observational features to be the most variable genes from the proliferation regulation category in the Gene Ontology \citep{ashburner2000gene}, which are significant modulators of cancer activity such as oncogenes and tumor suppressor genes.
We compute the sample precision matrices $\hTheta_0, \hTheta_1, \ldots, \hTheta_{10}$ and use them as input to \rref{algm:iterative-difference-projection-noisy} with $\gamma = 0.99$.

Given estimates $\{ \hatB_k \}_{k=1}^K, \hatH$ from our algorithm, we let $\tildeB_k = M_k \odot \hatB_k$, where $(M_k)_{ij} = \kron_{\{ |(\hatB_k)_{ij}| > 0.04 \}}$; i.e., we truncate the entries away from zero.
%
%
We treat the resulting parameters as a new generative model.
This tests our method in a more realistic setting, with parameters based on real data, while retaining the ability to measure performance.
Since the entries of the matrices $\tildeB_k$ are smaller than for our synthetic data, we consider larger sample sizes $n \in \{ 1e6, 5e6, 1e7, 5e7, 1e8 \}$.
As seen in \rref{fig:semisynthetic-results}, \rref{appendix:real-data}, we successfully recover the generative model.

\textbf{Hypothetical Workflow.}
We illustrate a hypothetical workflow of our method on biological data.
If we run our algorithm for all mutations ($K = 83$) on the $p = 83$ most variable genes, we obtain the graph in \rref{fig:latent-graph-real-data}.
We see that G12R, G12V, G13R, G13V, and G12I all perturb highly-connected latent nodes with several descendants.
The G12 and G13 positions in the KRAS protein are key functional residues whose mutations are known to be causal drivers of cancer \citep{huang2021kras}.
This indicates that the learned graph can highlight influential biological pathways which may be useful for prioritizing therapeutic development.
The matrix $\hatH$ from our algorithm gives a mapping from genes to latent variables that can be transferred across datasets and related to other observations.
For example, we compute estimates of the latent variables for the 589 lung cancer patients in the Cancer Genome Atlas \citep{liu2018integrated} and relate these variables to the patients' survival outcomes.
We find that 5 latent variables, those targeted by the M170L, I163S, A55G, G12R, and Q25H mutations, are significantly associated with survival outcomes (at significance $0.1$, after multiple testing correction).
See \rref{appendix:real-data} for details.
Note that the output of our method should be treated as exploratory; further theoretical and methodological development is required to better match complex real-world data.

\section{Discussion}\label{sec:discussion}

In this paper, we showed that a latent causal model is identifiable when data is available from an intervention on each latent variable.
Conversely, we showed that, in the worst case, such data is necessary for identifiability of the latent representation.
Our proof is constructive, consisting of an algorithm for recovering the latent representation, which can be adapted to the noisy setting.
Our algorithm was developed for maximal clarity of our identifiability result, and leaves open several directions for future work.

\textbf{Theory of latent interventional Markov equivalence.} We established sufficient and (worst-case) necessary conditions for the \textit{complete} identifiability of the parameters $H$ and $\{ B_k \}_{k=0}^K$.
However, in many settings, it is expected that these parameters (and corresponding combinatorial structures) are only \textit{partially} identifiable.
Indeed, Proposition~\ref{prop:non-identifiability} suggests that the problem parameters may be partially recoverable.
In future work, it would be interesting to develop a theory of identifiability for $K < p$ interventions.

\textbf{Non-linear setting.} Our results require that both the latent linear structural equation model and the mixing function are linear.
We expect that the insights developed here may apply when one or both of these elements are non-linear.
Notably, the contemporaneous work of \citet{ahuja2022interventional} shows that, under certain conditions, the case of \textit{polynomial mixing} can be reduced to the case of linear, which can be immediately applied to extend our result.

\textbf{Statistical analysis of causal disentanglement.} A next step beyond identifiability is to investigate the statistical properties of the setup.
This includes lower bounds on the accuracy of recovering the parameters $H$ and $\{ B_k \}_{k=0}^K$, along with corresponding combinatorial structures such as $\cG$ and the matching between $k$ and $i_k$, and computationally efficient algorithms that match these lower bounds.
Our identifiability result suggests that the differences of precision matrices may play a role.
These differences appear in \citet{varici2021scalable}, which uses techniques for directly estimating differences between precision matrices.
Moreover, it would be interesting to develop a score-based approach, e.g., (penalized) maximum likelihood estimation.
Our problem formulation suggests a natural parameterization for such an approach, which could be addressed using combinatorial optimization or gradient-based search.


\section*{Acknowledgements}
We thank Nils Sturma, Mathias Drton, Jiaqi Zhang, and Alvaro Ribot for helpful discussions.
We thank our anonymous reviewers for helpful suggestions that improved the paper.
Chandler Squires was partially supported by  an  NSF  Graduate  Research  Fellowship.
Anna Seigal was supported by the Society of Fellows at Harvard University.
Salil Bhate was supported by the Eric and Wendy Schmidt Center at the Broad Institute.
In addition, this work was supported by NCCIH/NIH (1DP2AT012345), ONR (N00014-22-1-2116), NSF (DMS-1651995), the MIT-IBM Watson AI Lab, and a Simons Investigator Award to Caroline Uhler.

\bibliography{bib}
\bibliographystyle{icml2023}

\clearpage
\let\addcontentsline\originaladdcontentsline
\onecolumn
\renewcommand{\contentsname}{Contents of Appendix}
\addtocontents{toc}{\protect\setcounter{tocdepth}{2}}
\tableofcontents

\newpage
\appendix
\onecolumn

\section{Additional related work}\label{appendix:additional-related}

\rref{fig:assumptions} shows the two graphical conditions assumed in some prior works.

\begin{figure}[h]
    \centering
    \begin{subfigure}{.27\textwidth}
        \includegraphics[width=\textwidth]{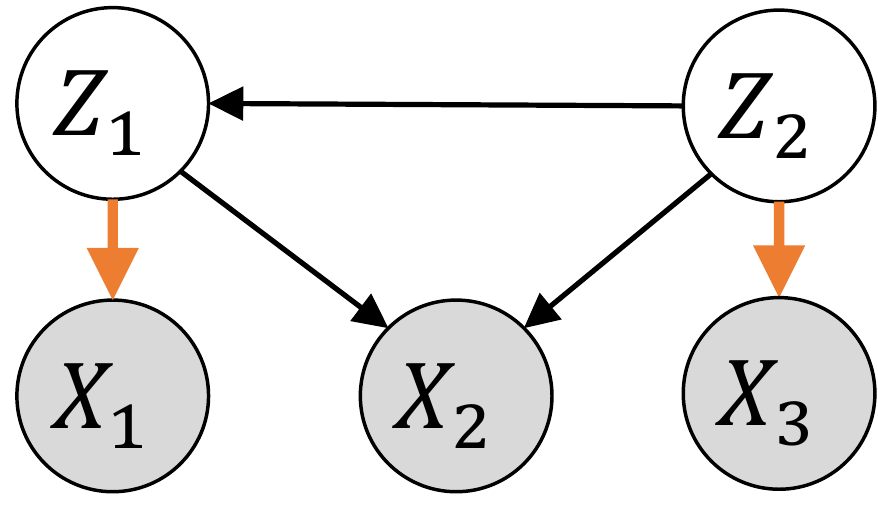}
        \caption{Pure children}
    \end{subfigure}
    \quad\quad
    \begin{subfigure}{.27\textwidth}
        \includegraphics[width=\textwidth]{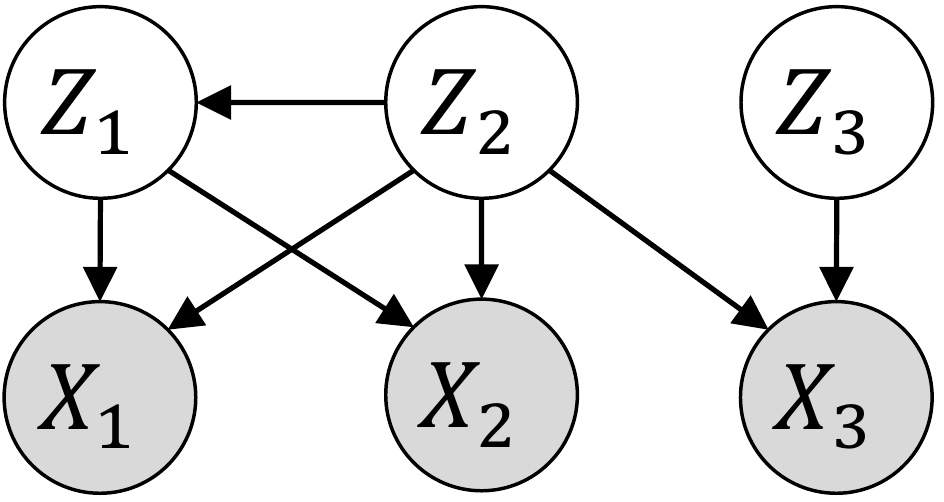}
        \caption{No twins}
    \end{subfigure}
    \caption{
    \textbf{Graphical conditions assumed in prior works.}
    In \textbf{(a)}, the orange edges link pure children ($X_1$ and $X_3$) to their parents ($Z_1$ and $Z_3$, respectively).
    In \textbf{(b)}, the no twins assumption is satisfied since the observed children of $Z_1$, $Z_2$, and $Z_3$ are, respectively, $\{ X_1, X_2 \}$, $\{X_1, X_2, X_3\}$, and $\{X_3 \}$, and these three sets are distinct.
    }
    \label{fig:assumptions}
\end{figure}
\section{Non-generic soft interventions}\label{appendix:unfaithful-counterexample}

We discuss the genericity condition in \rref{assumption:data-generating}(b). 
We show that for soft interventions in which this genericity condition fails to hold, identifiability of the causal disentanglement problem as in \rref{thm:G_and_Gbar} may fail.
The following matrices satisfy all of \rref{assumption:data-generating}, except for the genericity condition in \rref{assumption:data-generating}(b), since $B_1^\top \be_1 = 2 B_0^\top \be_1$:
\[
B_0 = \begin{bmatrix}
    1 & 1
    \\
    0 & 1
\end{bmatrix},
\qquad
B_1 = \begin{bmatrix}
    2 & 2
    \\
    0 & 1
\end{bmatrix},
\qquad
B_2 = \begin{bmatrix}
    1 & 1
    \\
    0 & 2
\end{bmatrix},
\qquad
G = \begin{bmatrix}
    1 & 0
    \\
    0 & 1
\end{bmatrix}.
\]
Consider the alternative matrices
\[
\hatB_0 = \begin{bmatrix}
    1 & 0
    \\
    0 & 1
\end{bmatrix},
\qquad
\hatB_1 = \begin{bmatrix}
    2 & 0
    \\
    0 & 1
\end{bmatrix},
\qquad
\hatB_2 = \begin{bmatrix}
    1 & 0
    \\
    0 & 2
\end{bmatrix},
\qquad
\hatG = \begin{bmatrix}
    1 & -1
    \\
    0 & 1
\end{bmatrix}.
\]
These do not differ from the original tuple of matrices via a permutation.
However, one can check that they are a valid solution, since
\begin{align*}
    \Theta_0 
    &=
    \hTheta_0
    =
    \begin{bmatrix}
        1 & 1
        \\
        1 & 2
    \end{bmatrix},
    \qquad 
    \Theta_1 
    =
    \hTheta_1
    =
    \begin{bmatrix}
        4 & 4
        \\
        4 & 5
    \end{bmatrix},
    \qquad 
    \Theta_1 
    =
    \hTheta_1
    =
    \begin{bmatrix}
        1 & 1
        \\
        1 & 5
    \end{bmatrix}.
\end{align*}
\section{Pseudoinverse of a covariance matrix}\label{appendix:precision-matrix}

\begin{prop}\label{prop:pseudoinverse-precision}
Let $X = G Z$ for $G \in \bbR^{p \times d}$ with full column rank.
Assume $\Cov(Z)$ is invertible and let $K := \Cov(Z)^\inv$.
Then $\Cov(X)^\dagger = H^\top K H$, 
where $H := G^\dagger$.
\end{prop}
\begin{proof}
The covariance matrices $\Cov(X)$ and $\Cov(Z)$ are related via $\Cov(X) = G \cdot \Cov(Z) \cdot G^\top$.
The property $(UV)^\dagger = V^\dagger U^\dagger$ holds whenever $U$ has full column rank and $V$ has full row rank \citep{greville1966note}.
The matrix $G$ has full column rank, $\Cov(Z)$ has full rank, and $G\T$ has full row rank.
Hence $\Cov(X)^\dagger = (G\T)^\dagger \Cov(Z)^\dagger G^\dagger = H^\top K H$.
%
\end{proof}

\section{The partial order RQ decomposition}\label{appendix:partial-order-rq}

Recall the partial order RQ decomposition from \rref{defn:partial-order-rq-decomposition}.
We present \rref{algm:partial-order-rq} to find the partial order RQ decomposition of a matrix. In \rref{line:normalize}, the $\normalize$ operator is $\normalize(\bv) := \frac{\bv}{\| \bv \|_2}$.
We let $\bzero_{d \times p}$ denote the $d \times p$ matrix of zeros and $Q_{j \succeq i}$ denote the submatrix of $Q$ on rows $j$ with $j \succeq i$.
We say that a partial order $\prec$ is consistent with the total order $1, 2, \ldots, d$ if $i \prec j$ implies $i < j$.
Any partial order can be put in this form by relabelling.

\begin{prop}
Let $H \in \bbR^{d \times p}$ be full rank and fix a partial order $\prec$ over $[d]$.
Then there exists a unique partial order RQ decomposition of $H$.
If $\prec$ is consistent with the total order $1, 2, \ldots, d$, the decomposition is returned by \rref{algm:partial-order-rq}.
\end{prop}

\begin{proof}
The matrix $Q_{j \succeq i}$ has fewer rows than columns and $\bh_i \in \langle \bq_j : j \succeq i \rangle$, by its construction in \rref{algm:partial-order-rq}. Hence the vector $\br$ of non-zero entries in the $i$-th row of $R$ is the unique solution to $Q_{j \succeq i}\T \br = \bh_i$.
By construction, we see that $H = R Q$ and, moreover, $R_{ij} = 0$ if $j \not\succeq i$, and $\bq_i$ is orthogonal to $\bq_j$ for $i \prec j$.
Furthermore, the entry $R_{ii}$ is positive, since $\bq_i$ is the (normalized) projection of $\bh_i$ onto $W_i$.
\end{proof}

\begin{algorithm}[htbp]
	\caption{Partial Order RQ Decomposition}\label{algm:partial-order-rq}
	\begin{algorithmic}[1]
		\STATE \textbf{Input:} Matrix $H \in \bbR^{d \times p}$, partial order $\prec$ over $[d]$ consistent with the total order $1, 2, \ldots, d$
		\STATE \textbf{Output:} A partial order RQ decomposition $R$, $Q$
		\STATE Let $R = \bzero_{d \times d}$, $Q = \bzero_{d \times p}$
		\FOR{$i = 1, \ldots, d$}
		    \STATE Let $\bh_i$ be the $i$-th row of $H$
		    \STATE Let $W_i = \langle \bq_j : j \succ i \rangle$
		    \STATE Let $\bq_i = \normalize(\proj_{W_i^\perp} \bh_i)$ be the $i$-th row of $Q$ \label{line:normalize}
		    \STATE Let $\br = (Q_{j \succeq i}\T)^\dagger \bh_i$
		    \label{line:linear-solution}
		    \STATE For $i \preceq j$, let $R_{ij} = r_j$
		\ENDFOR
		\STATE \textbf{return} $R$, $Q$
	\end{algorithmic}
\end{algorithm}
\section{Further preliminaries for identifiability and reduction}\label{appendix:further-preliminaries}

We prove \rref{fact:rank-one-decomposition} and discuss the reduction described in \rref{sec:reduction}.

\begin{proof}[Proof of \rref{fact:rank-one-decomposition}]
    We have $B^\top \! B= B^\top \! \! \left( \sum_{i=1}^d \be_i \be_i^\top \right) \! B = \sum_{i=1}^d B^\top \! \! \left( \be_i \be_i^\top \right) \! B = \sum_{i=1}^d (B^\top \be_i )^{\otimes 2}$.  
\end{proof}

\begin{prop}\label{prop:difference-in-contexts}
    Consider the setup in \rref{assumption:data-generating}.
    For $k, \ell \in [K]$, we have
    \begin{equation}\label{eqn:difference-latent-precision}
        B_k^\top B_k - B_\ell^\top B_\ell
        = 
        (B_k^\top \be_{i_k})^{\otimes 2} - (B_0^\top \be_{i_k})^{\otimes 2}
        -
        (B_\ell^\top \be_{i_\ell})^{\otimes 2} + (B_0^\top \be_{i_\ell})^{\otimes 2},
    \end{equation}
    and thus
    \begin{equation*}
        \Theta_k - \Theta_\ell = 
        (H B_k^\top \be_{i_k})^{\otimes 2} - (H B_0^\top \be_{i_k})^{\otimes 2}
        -
        (H B_\ell^\top \be_{i_\ell})^{\otimes 2} + (H B_0^\top \be_{i_\ell})^{\otimes 2}.
    \end{equation*}
\end{prop}
\begin{proof}
    The proof follows the same steps as \rref{prop:key-identity}, together with the fact that
    \begin{align*}
        B_k^\top B_k - B_\ell^\top B_\ell
        &=
        (B_k^\top B_k - B_0^\top B_0)
        -
        (B_\ell^\top B_\ell - B_0^\top B_0). \qedhere
    \end{align*}
\end{proof}

\section{Reduction}\label{appendix:reduction}

In this section, we show that a more general causal disentanglement problem can be simplified to one that satisfies our assumptions.
We consider an unknown observational context, and multiple contexts with the same target.
We focus here on the case of perfect interventions.


\subsection{Reducing to one intervention per node}

We identify which contexts correspond to interventions on the same node.
Thus, we can reduce to the case of one intervention per node by removing any redundant contexts.
We do not use knowledge of which context is the observational context here.

\begin{prop}\label{prop:identify-same-target}
    Consider the setup in Assumptions \ref{assumption:data-generating} and \ref{assumption:perfect-interventions}.
    Assume generic parameters for $B_0$, $\lambda_k$, and $\lambda_\ell$.
    For $k, \ell \in [K]$, we have $r_{k,\ell} = 1$ if and only if $i_k = i_\ell$, where $r_{k, \ell}:= \rank(\Theta_k - \Theta_\ell)$.
\end{prop}
\begin{proof}
    Since $H$ is full rank, we have $r_{k,\ell} = \rank(B_k^\top B_k - B_\ell^\top B_\ell)$.
    Thus, we consider $\rank(B_k^\top B_k - B_\ell^\top B_\ell)$.
    Suppose $i_k = i_\ell = i$.
We have
    \[
    B_k^\top B_k - B_\ell^\top B_\ell
    =
    \left( B_k^\top \be_i \right)^{\otimes 2}
    -
    \left( B_\ell^\top \be_i \right)^{\otimes 2},
    \] 
    by \rref{eqn:difference-latent-precision} in \rref{prop:difference-in-contexts}.
    Both $B_k^\top \be_i$ and $B_\ell^\top \be_i$ have a single nonzero entry at the $i$-th coordinate,
    by \rref{assumption:perfect-interventions}.
    Thus $r_{k,\ell} = 1$.

    Suppose $i_k \neq i_\ell$ and assume without loss of generality that $i_k < i_\ell$.
    Given a matrix $M$, let $M_{U}$ denote the submatrix of $M$ with rows and columns indexed by the elements of the set $U$.
    We have
    \[
    \left(
    B_k^\top B_k - B_\ell^\top B_\ell
    \right)_{ \{ i_k,i_\ell \}}
    =
    \begin{bmatrix}
        \lambda_k^2 - (B_0)_{i_k, i_k}^2 & -(B_0)_{i_k, i_k} (B_0)_{i_k, i_\ell}
        \\
        -(B_0)_{i_k, i_k} (B_0)_{i_k, i_\ell} & -\lambda_\ell^2 + (B_0)_{i_\ell, i_\ell}^2 - (B_0)_{i_k,i_\ell}^2
    \end{bmatrix},
    \]
    by \rref{eqn:difference-latent-precision} in \rref{prop:difference-in-contexts} and \rref{assumption:perfect-interventions}.
    For generic parameters, this submatrix has rank two, so the full matrix has rank at least two; i.e., $r_{k,\ell} \geq 2$.
\end{proof}

\subsection{Reducing to a known observational context}

The previous section explains how to reduce to the case with one intervention per latent node.
We may also reduce to the case with only one observational context: if more than one context is the observational context, they will all have the same inverse covariance matrix, so we may select only one of these contexts to serve as the observational context $k = 0$.
Next we show that, with one intervention per node, and one observational context, we can identify the observational context.
We show that the observational context has the ``sparsest" changes from the other contexts.
We formalize this intuition with the following definition.

\begin{defn}
The \emph{deviation score} of context $k$ is 
\[
r_k := \sum_{\ell \in [K] \setminus \{ k \}} r_{k,\ell},
\]
where $r_{k,\ell} := \rank(\Theta_k - \Theta_\ell)
for all~k, \ell \in [K]$.
%
\end{defn}

\begin{prop}\label{prop:identify-observational}
    Consider the setup in \rref{assumption:data-generating}.
Then $k^* \in \{ 0 \} \cup [K]$ is an observational context if and only if $k^* = \arg\min_{k \in \{ 0 \} \cup [K]} r_k$.
\end{prop}
\begin{proof}
    Let $\source(\cG)$ denote the set of source nodes in $\cG$.
    By \rref{prop:key-identity}, $r_{0,\ell} = 1 + \kron_{\pa(i_\ell) \neq \varnothing}$ for all $\ell \in [K]$.
    Thus, $r_0 = 2K - |\source(\cG)|$.
    
    For $k \neq 0$, we have $r_{k,\ell} \geq 2 + \kron_{\pa(i_\ell) \neq \varnothing}$ for all $\ell \in [K] \setminus \{ k \}$.
    Thus, $2K$.
    Since $\cG$ must have at least one source node, we see that $r_k > r_0$ for all $k \neq 0$.
\end{proof}
\section{Hypothesis testing a necessary condition for model membership}\label{appendix:hypothesis-testing}

We define the null hypothesis
\begin{equation*}
    H_0: \rank(\Theta_k - \Theta_0) \leq 2
    \qquad
    \forall~k \in [K]
\end{equation*}

\rref{assumption:data-generating}(b) implies that $H_0$ holds, by
\rref{prop:key-identity}.
The null hypothesis $H_0$ is a necessary condition for membership of $(\Theta_0, \Theta_1, \ldots, \Theta_K)$ in the model defined by \rref{assumption:data-generating}.
However, $H_0$ is not a sufficient condition for model membership: we may have $\rank(\Theta_k - \Theta_0) \leq 2$ for all $k \in [K]$, despite some interventions not targeting single nodes.
For example, if $\cG$ is the empty graph, and all interventions have two targets, then $H_0$ holds.
These cases may be ruled out with other conditions implied by model membership.
We leave a membership test for our model to future work.
Here, we focus on developing a test for $H_0$.

Prior work on testing latent variables models \cite{drton2007algebraic,squires2022causal} use such rank constraints.
To test whether a matrix $M \in \bbR^{p \times p}$ is rank $k$, one can test that all minors of size $k+1$ vanish; i.e., the collection of hypotheses
\[
H_{A,B}: t_{A, B} = 0
\qquad
A, B \subseteq [p], |A| = |B| = k+1
\]
where $t_{A,B} := \det(M_{A,B})$.

For example, \citet{squires2022causal} use this to test whether certain submatrices of a covariance matrix are rank one, as follows.
Let $\hatM$ be the sample covariance matrix computed from $n$ samples.
If the underlying distribution is multivariate Gaussian, it is well-known that $\hatM$ follows a Wishart distribution.
Now, for each pair of subsets $A, B$, compute the empirical minor $\hatt_{A,B} := \det(\hatM_{A, B})$.
Then, compute an estimate $\widehat{\Var}(t_{A,B})$.
Such an estimate can be obtained by evaluating the expression for $\Var(t_{A,B})$ in \citet{drton2008moments}, which characterizes the moments of minors for Wishart matrices.
Given this estimate, compute the z-score $z_{A,B} = \hatt_{A,B} / \sqrt{\widehat{\Var}(t_{A,B})}$.
By typical asymptotic theory, $z_{A,B} \to \cN(0,1)$ as $n \to \infty$, so we can use the z-score to compute an asymptotically correct p-value.
Finally, the p-values for all pairs of subsets $A, B$ can be aggregated into a single p-value using a multiple hypothesis testing procedure such as Bonferroni correction or Sidak adjustment.

In principle, a similar procedure can be performed to test our null hypothesis $H_0$.
However, under a Gaussianity assumption, $\Theta_k$ and $\Theta_0$ follow \textit{inverse} Wishart distributions, rather than Wishart distribution.
This would require expressions for the moments of minors for \textit{inverse} Wishart matrices.
We leave such a hypothesis test, which could give guarantees on false discovery rate control, to future work.
Instead, we demonstrate the performance of a simple hypothesis test based on the singular values of the matrix $\hTheta_k - \hTheta_0$.
Let
\[
\rho_2(M) := 
\left( \sigma_1^2(M) + \sigma_2^2(M) \right) / \left( \sum\nolimits_{i=1}^p \sigma_i^2(M) \right)
\]
If $\rank(M) \leq 2$, then $\rho_2(M) = 1$, otherwise, $\rho_2(M) < 1$.
Thus, we may test $H_0$ by checking where $\rho_2(\hTheta_k - \hTheta) > \tau$ for some threshold $\tau$ near 1.

We demonstrate the performance of this procedure for testing model membership.
We generate 500 random models following \rref{assumption:data-generating}, using the same hyperparameters as in \rref{sec:experimental}.
These models satisfy $H_0$.
We also generate 500 random models where the interventions target two nodes instead of one.
For each $k$, we pick intervention targets $I_k \subset [d]$ with $|I_k| = 2$, uniformly at random among all subsets of size two.
We hold all other hyperparameters of the simulation fixed.
We consider only $n = 2500$ samples, the smallest sample size used in \rref{sec:experimental}, and vary the threshold $\tau$ from 0.97 to 0.999, linearly spaced over 20 values.
The singular value based test is able to determine model membership at a rate well above random guessing, see \rref{fig:rank-test-roc}.

\begin{figure}
    \centering
    \includegraphics[width=0.4\textwidth]{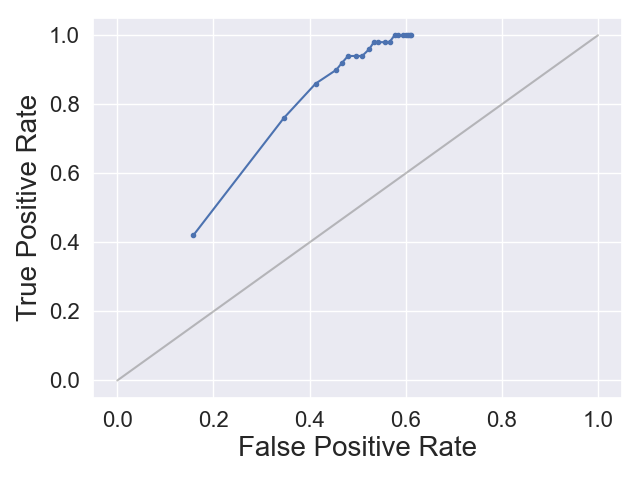}
    \caption{\textbf{Performance of our singular-value-based hypothesis test for $H_0$, the hypothesis that all precision matrix differences are of rank at most two.} The gray line indicates the performance of randomly guessing. Our test performs thresholding on the value $\rho_2(\hTheta_k - \hTheta_0)$. Here, we vary the threshold from 0.97 to 0.999.}
    \label{fig:rank-test-roc}
\end{figure}

\section{Computational complexity}\label{appendix:computational-complexity}

\rref{algm:iterative-difference-projection} takes as input $K$ precision matrices in $\bbR^{p \times p}$.
If these are each computed from at most $n$ samples, then the total cost is $\cO(K p^3 + K n p^2)$.
\rref{algm:identify-partial-order} runs for $K$ rounds and computes at most $2 K$ projections per round.
Each projection costs $\cO(p^3)$, so the cost of this step is $\cO(K^2 p^3)$.
In the remainder of \rref{algm:iterative-difference-projection}, we perform $\cO(K)$ matrix multiplications and Cholesky decompositions.
Each matrix multiplication costs $\cO(p^2)$ and each Cholesky decomposition costs $\cO(p^3)$.
The other operations (e.g. selecting rows) in \rref{algm:iterative-difference-projection} are negligible.
Therefore, the overall runtime of \rref{algm:iterative-difference-projection} is dominated by \rref{algm:identify-partial-order}, with a total cost of $\cO(K^2 p^3)$.




\section{Proofs}\label{appendix:proofs}

\subsection[Proof of non-identifiability for one missing intervention]{Proof of \rref{prop:non-identifiability}}

\begin{proof}
For (i), it suffices to find $\hatH$ and $\{ \hatB_k \}_{k=1}^K$ such that $\hatB_k \hatH = B_k H$ for all $k \in [K]$, and such that there is no $\sigma \in S(\cG)$ satisfying $\hatB_k = P_\sigma B_k P_\sigma\T$, $H = P_\sigma H$, by \rref{eqn:theta-k}.
Suppose $i_k \neq 1$ for any $k$.
Let
\[
\hatB_k =
\begin{bmatrix}
\text{-----}~~\be_1~~\text{-----}
\\
(B_k)_{2:d,1:d}
\end{bmatrix},
\quad\quad
\hatH =
\begin{bmatrix}
(B_0)_1^\top H
\\
H_{2:d,1:p}
\end{bmatrix}.
\]
Then, for all $k$, we have
$\hatB_k \hatH = B_k H$.
Suppose that $Z_1$ has at least one parent, i.e., $(B_0)_{1j} \neq 0$ for at least one $j > 1$.
Then the first row of $\hatH$ is a nonzero combination of at least two rows of $H$.
Hence it is not equal to a single row of $H$, since $H$ is full rank.
Thus, $\hatH$ is not equal to $H$ up to any permutation of rows.

For (ii), observe that for the stated example, the partial order $\prec$ given by $\hatB_0$ differs in general from the partial order $\prec_\cG$, since $Z_1$ has no predecessors in $\prec$.
\end{proof}
\section{Non-identifiability for imperfect interventions}\label{appendix:soft-intervention-counterexample}

\subsection{Parameter non-identifiability}
In this section, we show that \rref{assumption:perfect-interventions} is necessary to identify $H$.
Let
\begin{align*}
    B_0 &= \begin{bmatrix}
    (B_0)_{11} & (B_0)_{12}
    \\
    0 & (B_0)_{22}
    \end{bmatrix},
    \quad\quad
    B_1 = \begin{bmatrix}
    (B_1)_{11} & (B_1)_{12}
    \\
    0 & (B_0)_{22}
    \end{bmatrix},
    \\
    B_2 &= \begin{bmatrix}
    (B_0)_{11} & (B_0)_{12}
    \\
    0 & (B_2)_{22}
    \end{bmatrix},
    \quad\quad
    H = \begin{bmatrix}
    1 
    & 
    H_{12} 
    \\
    0 & 1
    \end{bmatrix}.
\end{align*}
Then, for any value $\hatH_{12} \in \bbR$, we have $B_k H = \hatB_k \hatH$ for all $k$, where
\begin{align*}
    \hatB_0 &= \begin{bmatrix}
    (B_0)_{11} & (B_0)_{12} + (B_0)_{11} H_{12} - (B_0)_{11} \hatH_{12}
    \\
    0 & (B_0)_{22}
    \end{bmatrix},
    \qquad 
    \hatB_1 = \begin{bmatrix}
    (B_1)_{11} & (B_1)_{12} + (B_1)_{11} H_{12} - (B_1)_{11} \hatH_{12}
    \\
    0 & (B_0)_{22}
    \end{bmatrix},
    \\
    \hatB_2 &= \begin{bmatrix}
    (B_0)_{11} & (B_0)_{12} + (B_0)_{11} H_{12} - (B_0)_{11} \hatH_{12}
    \\
    0 & (B_2)_{22}
    \end{bmatrix},
    \qquad
    \hatH = \begin{bmatrix}
    1 
    &
    \hatH_{12}
    \\
    0 & 1
    \end{bmatrix}.
\end{align*}

\subsection{Graph non-identifiability}
Suppose that $\cG$ has edges $2 \to 1$, $3 \to 2$, and $3 \to 1$.
Then the weight matrices have the form
\begin{align*}
    B_0 &= \begin{bmatrix}
    (B_0)_{11} & (B_0)_{12} & (B_0)_{13}
    \\
    0 & (B_0)_{22} & (B_0)_{23}
    \\
    0 & 0 & (B_0)_{33}
    \end{bmatrix},
    \quad\quad
    B_1 = \begin{bmatrix}
    (B_0)_{11} & (B_0)_{12} & (B_0)_{13}
    \\
    0 & (B_1)_{22} & (B_1)_{23}
    \\
    0 & 0 & (B_0)_{33}
    \end{bmatrix},
    \\
    B_2 &= \begin{bmatrix}
    (B_0)_{11} & (B_0)_{12} & (B_0)_{13}
    \\
    0 & (B_2)_{22} & (B_2)_{23}
    \\
    0 & 0 & (B_0)_{33}
    \end{bmatrix},
    \quad\quad
    B_3 = \begin{bmatrix}
    (B_0)_{11} & (B_0)_{12} & (B_0)_{13}
    \\
    0 & (B_0)_{22} & (B_0)_{23}
    \\
    0 & 0 & (B_3)_{33}
    \end{bmatrix},
    \\
    H &= \begin{bmatrix}
    1 & H_{12} & H_{13}
    \\
    0 & 1 & H_{23}
    \\
    0 & 0 & 1
    \end{bmatrix}.
\end{align*}
Let $\hatcG$ be the DAG with edges $2 \to 1$ and $3 \to 2$.
We show that there exist $\hatH$  and matrices $\hatB_k$, following the support of $\hatcG$, such that $B_k H = \hatB_k \hatH$ for all $k$. 
Note that $\cG$ and $\hatcG$ have the same transitive closure.
Let $\hatH_{12} \in \bbR$ and let $\hatH_{13}, \hatH_{23}$ be a solution to the system of equations:
\begin{align*}
    \begin{bmatrix}
    (B_0)_{11} & (B_0)_{12} + (B_0)_{11} (H_{12} - \hatH_{12})
    \\
    (B_1)_{11} & (B_1)_{12} + (B_1)_{11} (H_{12} - \hatH_{12})
    \end{bmatrix}
    \begin{bmatrix}
    \hatH_{13}
    \\
    \hatH_{23}
    \end{bmatrix}
    =
    \begin{bmatrix}
    (B_0)_{11} H_{13} + (B_0)_{12} H_{23} + (B_0)_{13}
    \\
    (B_1)_{11} H_{13} + (B_1)_{12} H_{23} + (B_1)_{13}
    \end{bmatrix}
\end{align*}
This system generically has a solution, since for generic parameters the matrix on the left hand side is rank two. Then a solution, with all matrices $\hatB_k$ have vanishing entry $(1,3)$, is as follows.
\begin{align*}
    \hatB_0 &=
    \begin{bmatrix}
    (B_0)_{11} & (B_0)_{12} + (B_0)_{11} (H_{12} - \hatH_{12}) & 0
    \\
    0 & (B_0)_{22} & (B_0)_{23} + (B_0)_{22} (H_{23} - \hatH_{23})
    \\
    0 & 0 & (B_0)_{33}
    \end{bmatrix}
    \\
    \hatB_1 &=
    \begin{bmatrix}
    (B_1)_{11} & (B_1)_{12} + (B_1)_{11} (H_{12} - \hatH_{12}) & 0
    \\
    0 & (B_0)_{22} & (B_0)_{23} + (B_0)_{22} (H_{23} - \hatH_{23})
    \\
    0 & 0 & (B_0)_{33}
    \end{bmatrix}
    \\
    \hatB_2 &=
    \begin{bmatrix}
    (B_0)_{11} & (B_0)_{12} + (B_0)_{11} (H_{12} - \hatH_{12}) & 0
    \\
    0 & (B_2)_{22} & (B_2)_{23} + (B_2)_{22} (H_{23} - \hatH_{23})
    \\
    0 & 0 & (B_0)_{33}
    \end{bmatrix}
    \\
    \hatB_3 &=
    \begin{bmatrix}
    (B_0)_{11} & (B_0)_{12} + (B_0)_{11} (H_{12} - \hatH_{12}) & 0
    \\
    0 & (B_0)_{22} & (B_0)_{23} + (B_0)_{22} (H_{23} - \hatH_{23})
    \\
    0 & 0 & (B_3)_{33}
    \end{bmatrix}
\end{align*}

\section{Comparison to \texorpdfstring{\citet{liu2022weight}}{Liu et al., 2022}}\label{appendix:weight-variant-comparison}

We compare \rref{thm:main_id_non_constructive} to \citet{liu2022weight}.
We restate their main result for convenience and notation.

\begin{thm*}[Theorem 4.1 of \citet{liu2022weight}]
\, Let $Z = A_k Z + \varepsilon_k$
and $X = g(Z) + \varepsilon_x$. Let $\boldeta_k$ be the sufficient statistic for the distribution of $Z$ in environment $k$.
That is, $\boldeta_k = \VEC(\tTheta_k)$, where $\tTheta_k$ denotes the inverse covariance matrix of $Z$ in the $k$th setting and $\VEC$ denotes the vectorization of a matrix.
We assume that $\VEC$ ignores zeros and repetitions.
Assume that
\begin{itemize}
    \item[(i)] $\{ x \in \cX \mid \varphi_{\varepsilon_x}(x) = 0 \}$ has measure zero, where $\varphi_{\varepsilon_x}$ is the characteristic function for $\varepsilon_x$,
    
    \item[(ii)] $g$ is bijective, and
    
    \item[(iii)] There exists $K + 1$ environments such that the following matrix is invertible:
    \[
    L = \begin{bmatrix}
    	| & | & & |
    	\\
    	\boldeta_1 - \boldeta_0 & \boldeta_2 - \boldeta_0 & \ldots & \boldeta_K - \boldeta_0
    	\\
    	| & | & & |
    \end{bmatrix}.
    \]
\end{itemize}
Then we can recover $g$ up to permutation and scaling.
\end{thm*}

First, we show that (i) and (ii) hold in our setting.
Our assumption that $X = GZ$ for $G$ invertible guarantees (ii).
Our assumption that $X$ is a deterministic function of $Z$ corresponds to taking $\varepsilon_x \sim \delta_0$, i.e., $\varepsilon_x = 0$ with probability one.
The characteristic function is $\varphi_\varepsilon(t) = 1$, thus satisfying (i).

We now show that (iii) is only satisfied in our setting when the number of edges in the latent graph is at most $d$.
The vector $\VEC(\tTheta_k - \tTheta_0)$ is of length $d + |E|$, where $|E|$ is the number of edges in the graphical model defined by $\Theta_0$.
To be invertible, $L$ must be a square matrix, and hence we require $K \geq d + |E|$. 
If $|E| > d$, then $K > 2d$, and we must have an intervention target $i$ such that $i = i_k$ for at least three values of $k$.
We have $\tTheta_k = B_k^\top B_k$, and thus 
\[
\tTheta_k - \tTheta_1 
= 
( \lambda_k \be_{i_k})^{\otimes 2} 
- 
(B_0\T \be_{i_k})^{\otimes 2}
-
( \lambda_\ell \be_{i_1})^{\otimes 2} 
+ 
(B_0\T \be_{i_1})^{\otimes 2}
\]
Given $k_1$, $k_2$, and $k_3$ such that $i_{k_1} = i_{k_2} = i_{k_3} = i$, we see that $\boldeta_{k_1}, \boldeta_{k_2}$ and $\boldeta_{k_3}$ differ only at position $(i, i)$.
The space of vectors that differ in at most one entry is at most two-dimensional.
Thus $\boldeta_{k_1}$, $\boldeta_{k_2}$, and $\boldeta_{k_3}$ are not linearly independent, and $L$ is not invertible.

\begin{algorithm}[t]
	\caption{\IdentifyPartialOrderNoisy}\label{algm:identify-partial-order-noisy}
	\begin{algorithmic}[1]
	    \STATE \textbf{Hyperparameters:} $\gamma$
		\STATE \textbf{Input:} Precision matrices $(\Theta_0, \Theta_1, \ldots, \Theta_K)$
		\STATE \textbf{Output:} Factor $\hatQ$, partial order $\prec$
		\STATE Let $\cI_0 = \{ \}$, $\hatQ = \bzero_{d \times d}$
		\FOR{$t = 1, \ldots K$}
		    \STATE Let $M_k = \proj_{\hatQ_{t-1}^\perp} (\Theta_k - \Theta_0)$ for each $k \not\in \cI_{t-1}$
    		\STATE Let $\rho_k = \sigma_1^2(M_k) / (\sum_{i=1}^p \sigma_i^2(M_k))$ for each $k \not\in \cI_{t-1}$
    		\STATE Pick $k \in \argmax_{k \not\in \cI_{t-1}} \rho_k$
    		\label{line:dimension-test-partial-order-noisy}
    		\STATE Let $\hatbq_k, \cA = \IdentifyAncestorsNoisy(\Theta_k, \Theta_0, \{ \hatbq_i \}_{i \in \cI_{t-1}} ; \gamma)$
    		\STATE Add $a' \succ k$ for any $a' \succeq a$, $a \in \cA$
    		\STATE Let $\cI_t = \cI_{t-1} \cup \{ k \}$, $\hatQ_t = [\hatbq_k; \hatQ_{t-1}]$
		\ENDFOR
		\STATE \textbf{return} $\hatQ$, $\prec$
	\end{algorithmic}
\end{algorithm}
\begin{algorithm}[htbp]
	\caption{\IdentifyAncestorsNoisy}\label{algm:identify-ancestors-noisy}
	\begin{algorithmic}[1]
	    \STATE \textbf{Hyperparameters:} $\gamma$
		\STATE \textbf{Input:} $\Theta_k$, $\Theta_0$, $\{ \hatbq_i \}_{i \in \cI}$
		\STATE \textbf{Output:} Vector $\hatbq_k$, ancestor set $\cA$
		\STATE Let $\cA = \cI$
		\FOR{$i \in \cI$}
		    \STATE Let $W_{\neg i} = [\hatbq_i : j \in \cI \setminus \{ i \}]$
		    \STATE Let $M_{\neg i} = \proj_{W_i^\perp} (\Theta_k - \Theta_0)$
    		\STATE Let $\rho_k = \sigma_1^2(M_k) / (\sum_{i=1}^p \sigma_i^2(M_k))$ for each $k \not\in \cI_{t-1}$
    		\label{line:dimension-test-ancestors-noisy}
    		\STATE If $\rho_k \geq \gamma$, let $\cA = \cA \setminus \{ i \}$
    		\label{line:pick-top-singular-vector-ancestors-noisy}
		\ENDFOR
		\STATE Let $W = [\hatbq_a : a \in \cA ]$
		\STATE Let $M = \proj_{W^\perp} (\Theta_k - \Theta_0)$
		\label{line:m-definition-ancestors-noisy}
		\STATE Let $\hatbq_k$ be the (normalized) leading left singular vector of $M$
		\STATE \textbf{return} $\hatbq_k$, $\cA$
	\end{algorithmic}
\end{algorithm}

\section{Finite-sample algorithms}\label{appendix:noisy-algorithm}

\textbf{Matrix rank scoring.} 
In \rref{line:dimension-test-ancestors} of \rref{algm:identify-ancestors} and \rref{line:dimension-test-partial-order} of \rref{algm:identify-partial-order}, we check whether a subspace is rank one.
In the finite-sample setting, we represent these subspaces by matrices and measure how close the matrices are to rank one.

We use the score $\rho(M) := \sigma_1^2(M) / \left( \sum_{i=1}^p \sigma_i^2(M) \right)$, where $\sigma_i(M)$ is the $i$th largest singular value of $M$; 
$\rho(M)$ can be interpreted as the percentage of \textit{spectral energy} associated to the largest singular value of $M$.
Using this score to choose the next element of the partial order does not require hyperparameters, see \rref{line:dimension-test-partial-order-noisy} of \rref{algm:identify-partial-order-noisy}.
In contrast, using this score to prune the set of ancestors requires a hyperparameter to determine whether a matrix is close enough to rank one, see $\gamma$ in  \rref{line:dimension-test-ancestors-noisy} of \rref{algm:identify-ancestors-noisy}.
Larger values of $\gamma$ (e.g., 0.999) result in a more conservative algorithm and will output a denser latent graph, while smaller values of $\gamma$ (e.g., 0.8) result in more aggressive pruning of the latent graph.

\textbf{Picking $\bq_k$.}
The matrix $M$ in \rref{line:m-definition-ancestors-noisy} of \rref{algm:identify-ancestors-noisy} is not guaranteed to be rank one in the finite-sample case. 
We instead select the leading left singular vector of $M$. 

\textbf{Picking nonzero rows.}
In the finite-sample case, the matrix $\hatD_k$ will not usually have only one nonzero row, see \rref{line:pick-nonzero-row} of \rref{algm:iterative-difference-projection-noisy}.
We estimate the intervention target $i_k$ by picking the row of largest norm.
Since we assume that $i_k$ is distinct for distinct $k$, we maintain a set $\cN$ of candidate intervention targets and do not allow replicates.

\begin{algorithm}[tbp]
	\caption{\IterativeDifferenceProjectionNoisy}\label{algm:iterative-difference-projection-noisy}
	\begin{algorithmic}[1]
	    \STATE \textbf{Hyperparameters:} $\gamma$
		\STATE \textbf{Input:} Precision matrices $(\Theta_0, \Theta_1, \ldots, \Theta_K)$
		\STATE \textbf{Output:} $\hatH$, $(\hatB_0, \hatB_1, \ldots, \hatB_K)$
		\STATE Let $d = K$
		\STATE Let $\hatQ, \prec = \IdentifyPartialOrderNoisy( (\Theta_0, \Theta_1, \ldots, \Theta_K) ; \gamma)$
		\STATE Let $\hatC_k = \cholesky( (\hatQ^\dagger)^\top \Theta_k \hatQ^\dagger)$ for $k = 0, \ldots, K$
		\STATE Let $\cN = [p]$
		\STATE Let $\hatR = I_d$
		\FOR{$k = 1, \ldots, K$}
    		\STATE Let $\hatD_k = \hatC_k - \hatC_0$
    		\STATE Pick $\hati_k \in \argmax_{i \in \cN} \| (\hatD_k)_i \|_2$ 
    		\label{line:pick-nonzero-row-noisy}
    		\STATE Let $\cN = \cN \setminus \{ \hati_k \}$
    		\STATE Let $\hatR_{\hati_k} = (\hatD_k)_{\hati_k} + (\hatC_0)_{\hati_k}$
		\ENDFOR
		\STATE Let $\hatH' = \hatR \hatQ$
		\STATE Let $\hatH = \Lambda \hatH'$, for $\Lambda$ diagonal such that $\hatH$ satisfies the conditions on $H$ in \rref{assumption:data-generating}(c)
		\STATE Let $\hatB_0 = \cholesky( (\hatH^\dagger)^\top \Theta_0 \hatH^\dagger)$ 
		\STATE Let $\hatB_k = \hatB_0 + \be_{\hati_k} \left( |\hLambda_{\hati_k,\hati_k}| \be_{\hati_k} - \hatB_0^\top \be_{\hati_k} \right)^\top$ for $k = 1, \ldots, K$
		\label{line:compute-bk-noisy}
		\STATE \textbf{return} $\hatH$, $(\hatB_0, \hatB_1, \ldots, \hatB_K)$
	\end{algorithmic}
\end{algorithm}

\section{Code and Data}\label{appendix:code-and-data}

Our code can be found at 
\begin{center}
    \small
    \url{https://github.com/csquires/linear-causal-disentanglement-via-interventions}.
\end{center}

\subsection[Optimizing over permutations]{Optimizing over $S(\cG)$}\label{appendix:best-matching}

Consider a partial order $\prec_\cG$, a set of true intervention targets $i_1, \ldots, i_K$, and a set of estimated intervention targets $\hati_1, \ldots, \hati_K$. 
The integer linear program \eqref{eqn:ilp} computes the topological order $\pi^*$ consistent with $\prec_\cG$ that maximizes the number of agreements between $i_k$ and $\hati_k$.
The topological order $\pi^*$ can be recovered by letting $\pi^*(i) = j$ for the unique $j$ such that $A_{ij} = 1$.
The first two lines of constraints ensure this uniqueness, and that $\pi^*(i) \neq \pi^*(i')$ for $i \neq i'$.

The final line of constraints ensures that $\pi^*$ is consistent with $\prec_\cG$.
If $\pi^*$ is not consistent with $\prec_\cG$, then there exists $i, i', j$ such that $i \prec_\cG i'$, $A_{i'j} = 1$, and $A_{ij'} = 0$ for all $j' \leq j$, which violates the constraint $\sum_{j' \leq j} \left( A_{ij'} - A_{i'j'} \right) \geq 0$.

\begin{equation}\label{eqn:ilp}
    \begin{aligned}
        \max_{A_{ij} \in \{ 0, 1 \}^{d \times d}} &\sum_{k=1}^K A_{i_k \hati_k}
        \\
        \st~&\sum_{i=1}^d A_{ij} = 1 \quad\forall~j \in [d]
        \\
        &\sum_{j=1}^d A_{ij} = 1 \quad\forall~i \in [d]
        \\
        &\sum_{j' \leq j} \left( A_{ij'} - A_{i'j'} \right) \geq 0 \quad\forall~i \prec_\cG i', \forall~j \in [d]
    \end{aligned}
\end{equation}

\subsection{Additional information on real data}\label{appendix:real-data}

The scRNA-seq dataset of \citet{ursu2022massively} is available at \url{https://www.ncbi.nlm.nih.gov/geo/query/acc.cgi?acc=GSE161824}.
The TCGA dataset of \citet{liu2018integrated} is available at \url{https://gdc-hub.s3.us-east-1.amazonaws.com/download/TCGA-LUAD.survival.tsv} and \url{https://gdc-hub.s3.us-east-1.amazonaws.com/download/TCGA-LUAD.htseq_fpkm.tsv.gz}.

\textbf{Processing.} We use \texttt{EnrichR} \citep{kuleshov2016enrichr} to pick the $p = 100$ and $p = 83$ most variable genes in the proliferation regulation gene set from the Gene Ontology.
We use the values from the processed dataset from \citet{ursu2022massively}; the only additional processing removes cells which were assigned to synonymous mutations (i.e., those that do not change any amino acids and hence do not have structural effects).

\textbf{Semi-synthetic analysis.} Our algorithm recovers the problem parameters for the semi-synthetic data, 
see \rref{fig:semisynthetic-results}.

\textbf{Comparison to TCGA dataset.} Our survival analysis is performed using the Cox proportional hazards model from the \texttt{lifelines} package \citep{davidson2019lifelines}.
To correct for multiple hypothesis testing, we use the Benjamini-Hochberg procedure from the \texttt{statsmodels} package \citep{seabold2010statsmodels}.

\begin{figure*}
    \centering
    \begin{subfigure}{.32\textwidth}
        \includegraphics[width=\textwidth]{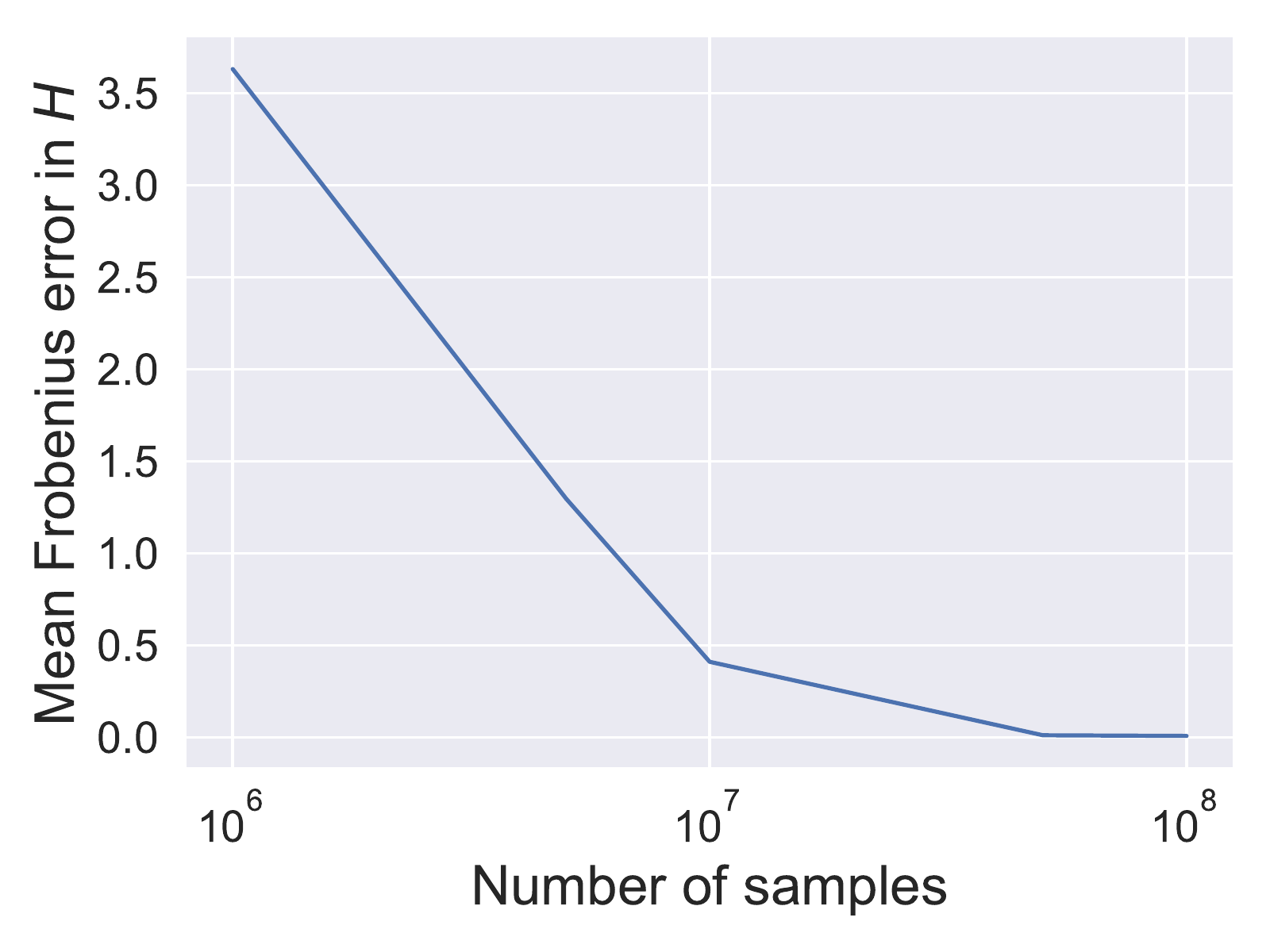}
        \caption{Error in estimating $H$}
    \end{subfigure}
    ~
    \begin{subfigure}{.32\textwidth}
        \includegraphics[width=\textwidth]{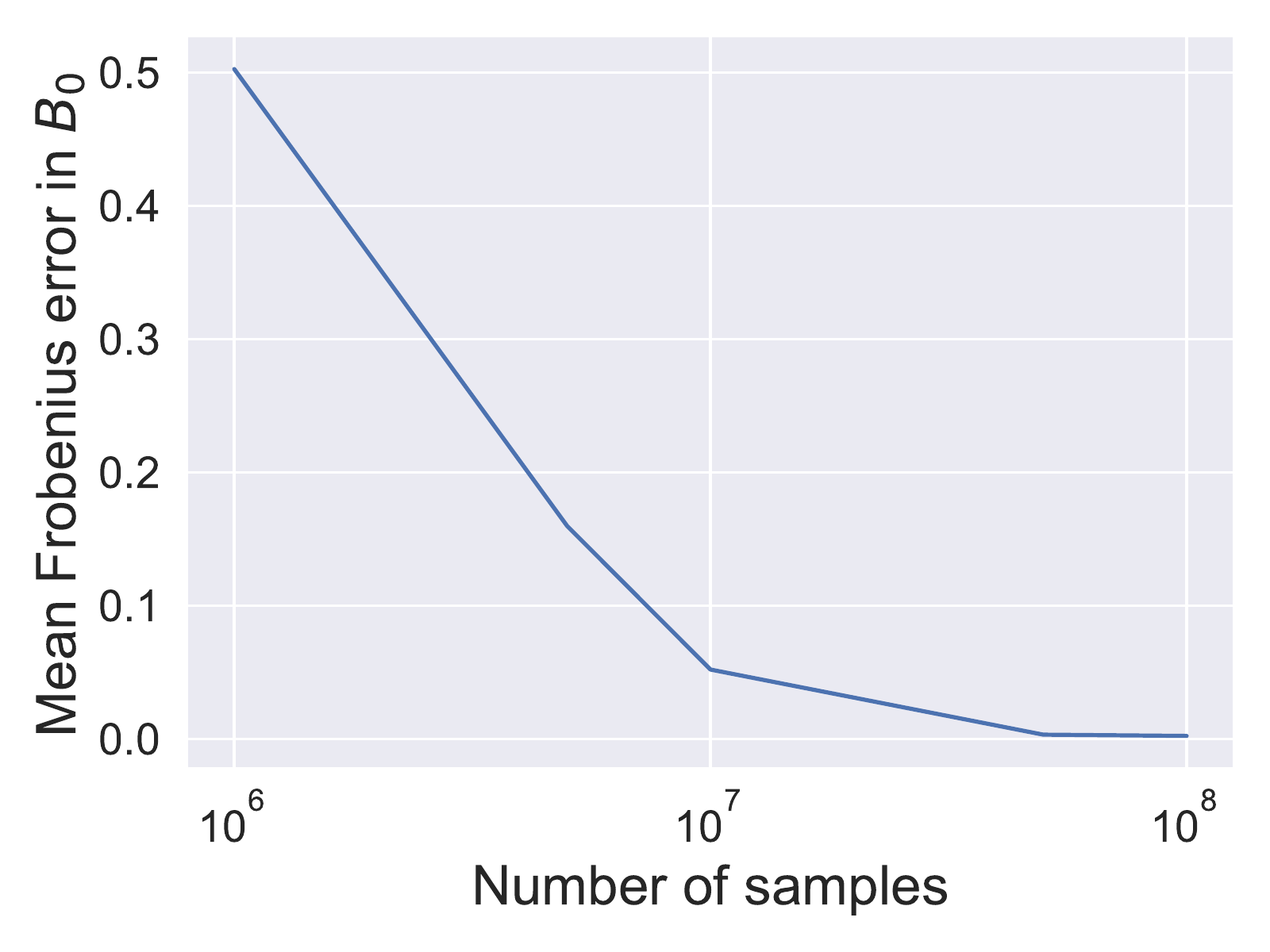}
        \caption{Error in estimating $B_0$}
    \end{subfigure}
    ~
    \begin{subfigure}{.315\textwidth}
        \includegraphics[width=\textwidth]{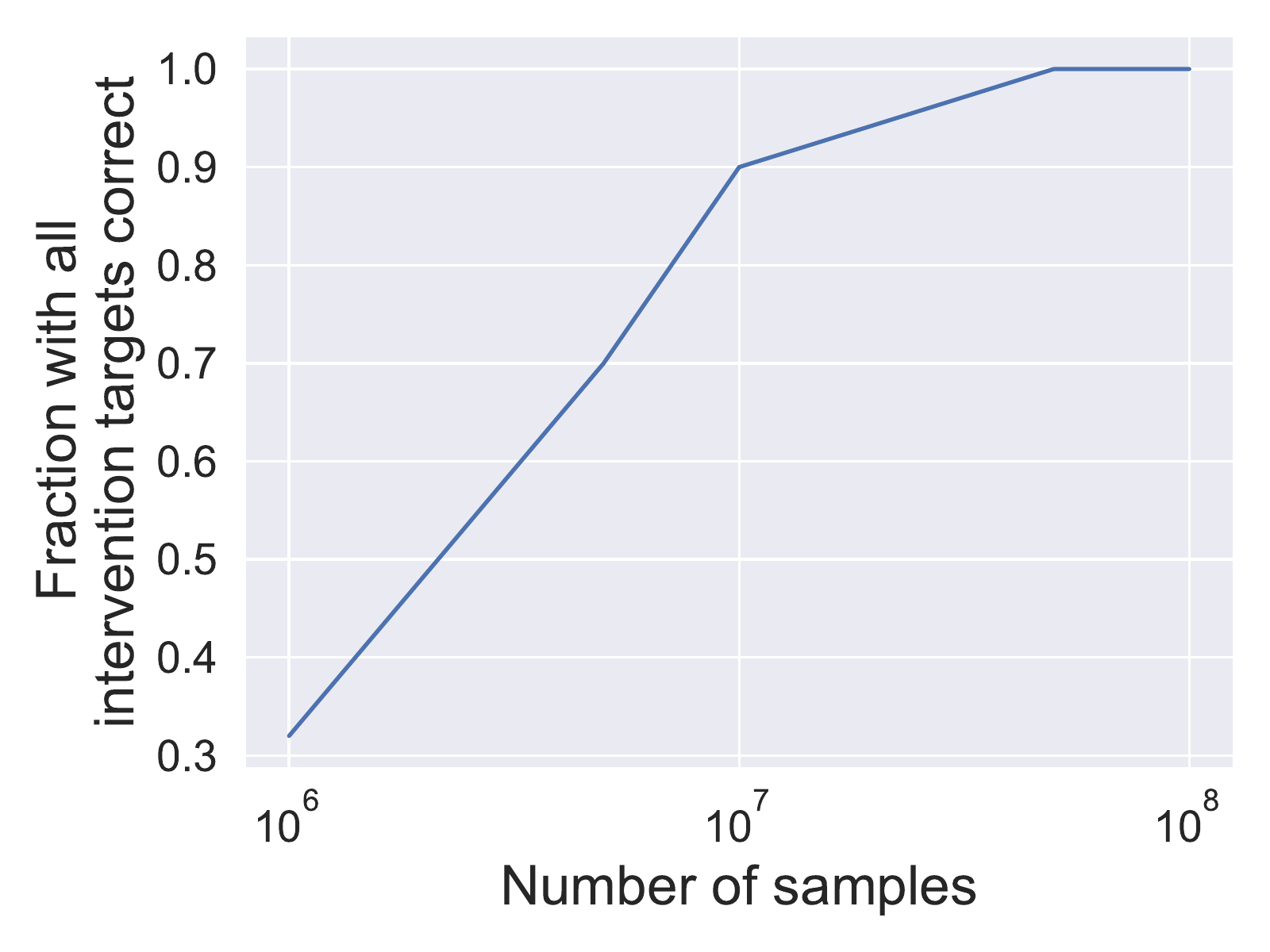}
        \caption{Intervention targets}
    \end{subfigure}
    \caption{
    \textbf{(Semi-synthetic) The adapted version of \rref{algm:iterative-difference-projection} is consistent for recovering $H$, $B_0$, and $\{ i_k \}_{k=1}^K$ from semi-synthetic data.}
    At each sample size, we generate 50 datasets.
    Note the logarithmic scale on the x-axis.
    In \textbf{(a)}, we plot the mean of $\| \hatH - H \|_2$, the error in Frobenius norm.
    In \textbf{(b)}, we plot the mean of $\| \hatB_0 - B_0 \|_2$.
    In \textbf{(c)}, we plot the fraction of models where all intervention targets were correctly estimated.
    }
    \label{fig:semisynthetic-results}
\end{figure*}

\end{document}